\documentclass{article}
\usepackage[nonatbib,preprint]{arxiv}
\usepackage{algorithmic}
\usepackage{amsthm}
\usepackage[utf8]{inputenc} % allow utf-8 input
\usepackage[T1]{fontenc}    % use 8-bit T1 fonts
\usepackage{hyperref} 
\hypersetup{
    colorlinks=true,
    citecolor=blue, 
    linkcolor=blue,
    filecolor=magenta,      
    urlcolor=blue,
linktocpage}
\usepackage{url}            % simple URL typesetting
\usepackage{booktabs}       % professional-quality tables
\usepackage{amsfonts}       % blackboard math symbols
\usepackage{nicefrac}       % compact symbols for 1/2, etc.
\usepackage{microtype}      % microtypography
\usepackage{amsmath}
\usepackage{amsthm}
\usepackage{wrapfig}
\usepackage{tablefootnote}
\usepackage[table]{xcolor}
\usepackage{enumitem}
\newtheorem{theorem}{Theorem}
\newtheorem{lemma}{Lemma}
\newtheorem{remark}{Remark}
\newtheorem{definition}[theorem]{Definition}
\newtheorem{assumption}[theorem]{Assumption}

\usepackage[ruled,vlined]{algorithm2e}
\usepackage{pbox}

\usepackage{comment}
\title{Online Double Oracle}
\usepackage{caption}
\usepackage{subcaption}
\usepackage{microtype}
\usepackage{graphicx}
\usepackage{booktabs} % for professional tables
\usepackage{amsmath}  %new add in
\usepackage{amsfonts}
\usepackage{hyperref}
\usepackage{amsmath,amsfonts,bm}

\def\eqref#1{equation~\ref{#1}}
\def\1{\bm{1}}

\def\va{{\bm{a}}}

\def\vc{{\bm{c}}}

\def\vl{{\bm{l}}}

\def\vpi{{\boldsymbol{\pi}}}

\def\mA{{\bm{A}}}

\def\mG{{\bm{G}}}

\DeclareMathAlphabet{\mathsfit}{\encodingdefault}{\sfdefault}{m}{sl}
\SetMathAlphabet{\mathsfit}{bold}{\encodingdefault}{\sfdefault}{bx}{n}

\usepackage{pbox}
\usepackage{pdfpages}

\author{Le Cong Dinh$^{*,1,2}$, Yaodong Yang$^{*,1,4}$, Stephen McAleer$^{5}$, Nicolas Perez-Nieves$^{3}$,\And  Oliver Slumbers$^4$, Zheng Tian$^4$, David Henry Mguni$^1$, Haitham Bou Ammar$^1$,  Jun Wang$^{1,4}$}

\begin{document}

\maketitle








\begin{abstract}
\vspace{-10pt}
Solving strategic games with huge action space is  a critical yet under-explored topic in economics, operations  research and artificial intelligence.  This paper\footnote{$^*$Equal contributions. Corresponding to <yaodong.yang@outlook.com>. $^1$Huawei U.K.  $^2$ University of Southampton (part of the work done during internship at Huawei). $^3$Imperial College London.  $^4$University College London. $^5$University of California, Irvine.} proposes new learning algorithms for solving two-player zero-sum normal-form games where the number of pure strategies is prohibitively large.  Specifically, we combine no-regret analysis from online learning with Double Oracle (DO) methods from game theory.~\footnote{Accepted at Transactions on Machine Learning Research (TMLR).} 
Our method -- \emph{Online Double Oracle (ODO)} -- is provably convergent to a Nash equilibrium (NE).  
Most importantly, unlike normal DO methods, ODO is \emph{rationale} in the sense that each agent in ODO can exploit strategic adversary with a regret bound of $\mathcal{O}(\sqrt{T k \log(k)})$  where $k$ is not the total number of pure strategies, but rather the size of \emph{effective strategy set} that is linearly dependent on the support size of the NE. On tens of different real-world games, ODO outperforms   DO,   PSRO methods, and no-regret algorithms such as Multiplicative Weight Update by a significant margin, both in terms of convergence rate to a NE and average payoff against strategic adversaries. 
\end{abstract}

%write the introduction focusing on double oracle, multiagent learning community.
%make it rigorous about the assumption, how to define {\vl}_t
%changing the the weight w
%e-best response in the algorithm 1
\vspace{-15pt}
\section{Introduction}
\vspace{-7pt}

% \yaodong{1) solving zero-sum games is an important topic. It has many real-world applications in game AI, adversarial attack etc. 2) we look at scalable appraoch to solve zero-sum games, where the number of pure strategies is huge or even open-ended. 3) introduce doube oracle's two-step process and PSRO and meta-games, with robot-detection example, and maybe one starcraft example.   4) highlight the weakness of these apporach: not theoretical grounded, cannot consider noise in feedback, and others.  5) hihglight our contribution..  BTW, don't mention the word of online MDP in the introudction at all.}\lecong{implemented}

%\lecong{new introduction start from here}
Understanding games with large action spaces is a critical  topic in a variety of fields including but not limited to economics, operations research  and artificial intelligence.   
%Applications include designing game AI \cite{vinyals2019grandmaster}, training Generative Adversarial Networks \cite{goodfellow2020generative}, and robust  control \cite{abdullah2019wasserstein}. 
A key challenge to solving large games is to compute a Nash equilibrium (NE) \cite{nash1950equilibrium} in which no player will be better off by deviation. 
Unfortunately, finding a NE is generally intractable in games; computing a two-player NE is known to be PPAD-hard   \cite{chen2006settling}. %\Haitham{English needs fixing in the previous paragraph. The 1st and 2nd paragraph are not connected.}
An exception is  two-player zero-sum games where the NE can be tractably solved by a linear program (LP) \cite{morgenstern1953theory}.  
Despite their polynomial-time solvability  \cite{van2020deterministic},  LP solvers are not adequate in games with prohibitively  large action spaces.  
As a result, researchers shift focus towards finding  approximation solutions \cite{mcmahan2003planning,lanctot2017unified} or  developing new solution concepts \cite{omidshafiei2019alpha,yang2020alphaalpha}. %, such as fictitious play (FP) \cite{browniterative}. % \cite{browniterative,zinkevich2007regret} %, such as fictitious play (FP)  and counterfactual regret minimisation (CFR) \cite{zinkevich2007regret}.   

Double Oracle (DO) algorithm \cite{mcmahan2003planning} and its variation Policy Space Response Oracles (PSRO)  \cite{lanctot2017unified} are powerful approaches to finding approximate NE in games where the support of a NE is relatively small.
%An important design principle that underpins many approximation solutions to  NE in large-scale two-player zero-sum (normal-form) games is the iterative best-response dynamics. %\stephen{DO algos are not best response dynamics, we should remove this sentence.}
%\lecong{An important design principle that underpins many approximation solutions to  NE in large-scale two-player zero-sum (normal-form) games is the best-response oracle}
%Two representative methods are double oracle (DO) \cite{mcmahan2003planning} and policy space response oracle (PSRO) \cite{lanctot2017unified}. 
%\stephen{Proposed sentence instead of the first two: "The tabular double oracle (DO) algorithm \cite{mcmahan2003planning} and its neural counterpart Policy Space Response Oracles (PSRO) \cite{lanctot2017unified} are powerful approaches to finding approximate Nash equilibria in games where the support of the Nash equilibrium is relatively small."}
In the dynamics of DO \cite{mcmahan2003planning},  players are  initialised with a limited number of strategies, thus playing only a sub-game of the original game; then, at each iteration, a best-response strategy to the NE of the last sub-game, which is assumed to be given by an \emph{Oracle},  will be added into each agent's strategy pool. The process stops when the best response is already in the strategy pool or the performance improvement becomes trivial.  
When an exact best-response strategy is not achievable, an approximate solution is often adopted. For example, PSRO methods \cite{lanctot2017unified,mcaleer2020pipeline,nieves2021modelling} applies reinforcement learning (RL) \cite{sutton2018reinforcement}  oralces to compute the best response.  
%Although in the worst case, DO/PSRO recovers to solve the original game, their effectiveness relies  on the fact that the support size of an NE is often much smaller than the game size \cite{mcmahan2003planning}. %\Haitham{There is so much talk before geting int the problem u want to solve in this paper.} %\cite{mcmahan2003planning}. 
 
%\yaodong{i think we missed an important method, which is the fictitious play. we need to say why it is not as good as double oracle to rule it out of the disucussion in the rest of the paper}\lecong{we use FP as a baseline in our experiment, so I think it is better to mention that in the related work}\zheng{I think this paragraph is well written but it may be more suitable to be placed in related work. To make the introduction concise and punctual we just need probably one sentence stating what DO/PSRO could do then focus on its limitations.}
%\dm{We seem to have neglected the fact that in this setting, the agents receive noisy feedback (and the matrix game is not known to them). As opposed to LP and FP which have neither of these features. I think this is a nice advantage worth emphasising} \lecong{changed. Mentioned it in the drawback of DO algorithms}

\begin{table*}[t!]
\vspace{-25pt}
\caption{{\small{Properties of existing  solvers on two-player zero-sum games $\mA_{n\times m}$. $^*$:DO in the worst case has to solve all sub-games till reaching the full game, so the time complexity is one order magnitude larger than LP. $^\dag$: Since PSRO uses approximate best response, the total time complexity is unknown. $^\ddag$ Note that the regret bound of ODO can not be directly compared with the time complexity of DO, which are two different notions.} }}
%\yaodong{discuss in the intro 0) comment on K can reach N, but usually it is smaller 1) FP is not no-regret 2) DO needs to know the full payoff}}
\vspace{-10pt}
\label{tbl:gamesolver}
\begin{center}
   \resizebox{1.\textwidth}{!}{
  \begin{tabular}{p{140pt} c  >{\centering}p{60pt} >{\centering}p{80pt} c c}
  \toprule
   \textbf{Method} & \pbox{50pt}{\textbf{Rational } \textbf{(No-regret)}} & \pbox{60pt}{\textbf{Allow $\epsilon$-Best Response}} & \pbox{80pt}{\textbf{No Need to Know the Full Matrix $\mA$}} & \pbox{100pt}{\textbf{Time Complexity ($\tilde{\mathcal{O}}$) /}  \textbf{ Regret Bound (${\mathcal{O}}$)}} & \textbf{Large Games} \\  \midrule 
      \pbox{120pt}{Linear Programming 
      \cite{van2020deterministic}}&  &  &  & 
      $\tilde{\mathcal{O}}\big(n \exp(-T/n^{2.38})\big)$& \\   \arrayrulecolor{black!30}\midrule 
 \pbox{140pt}{(Generalised) Fictitious Play 
\cite{leslie2006generalised}} & & \checkmark &  \checkmark & $\tilde{\mathcal{O}}\big(T^{-1/(n+m-2)}\big)$ & \\ \arrayrulecolor{black!30}\midrule 
 \pbox{120pt}{Multipli. Weight Update
\cite{freund1999adaptive}}& \checkmark &  & \checkmark &$\mathcal{O}\big(\sqrt{ \log(n)/T} \big)$ & \\ \arrayrulecolor{black!30}\midrule 
   \pbox{120pt}{ Double Oracle 
   \cite{mcmahan2003planning}}& & &  \checkmark  & $\tilde{\mathcal{O}}\big(n \exp(-T/n^{3.38})\big)^*$ & \checkmark \\ \arrayrulecolor{black!30}\midrule 
    \pbox{140pt}{Policy Space Response Oracle 
    \cite{lanctot2017unified}}&  & \checkmark &  \checkmark & $\times^\dag$  & \checkmark \\  \arrayrulecolor{black!30}\midrule \midrule 
     \pbox{120pt}{\textbf{Online Double Oracle}} & \checkmark & \checkmark & \checkmark & $\mathcal{O}\big(\sqrt{k \log(k)/T}\big)^\ddag$ & \checkmark \\ 
   \arrayrulecolor{black} \bottomrule
  \end{tabular}}
\end{center}
\vspace{-20pt}
\end{table*}
While DO/PSRO provides an efficient way to approximate the NE in large-scale zero-sum games, it still faces two open challenges. \textbf{Firstly}, DO methods require both  players to \emph{coordinate} in order to solve the NE in  sub-games (i.e., both players have to follow the same learning dynamics such as Fictitious  Play (FP)  \cite{brown1951iterative}). This contradicts many real-world scenarios where an opponent can play any (non-stationary) strategies in sub-games.  
\textbf{Secondly}, and most importantly, DO methods are not \emph{rational}  \cite{bowling2001rational}, in the sense that they do not provide a learning scheme that can exploit the adversary (i.e., achieving no-regret).
%the DO algorithm only focuses on finding the fixed point solutions of games but not a learning scheme a rational player can follow when facing unknown adversarys. 
%Learning in the presence of a strategic adversary is critical in many real-world problems. %our adversarys are not always another DO program\Haitham{Rewrite. Weak argument. Try to make it rigourours somehow.}.  
While a NE strategy guarantees the best performance in the worst scenario, it can be too pessimistic for a player to play such a strategy. %, especially when there are better ways to exploit its adversary. 
For example, in Rock-Paper-Scissors (RPS), playing the  NE of $(1/3, 1/3, 1/3)$ makes one player unexploitable. However,  %playing the NE strategy also disables the potential of  exploitation. 
if the adversary acts irrationally and addicts to one strategy, say ``Rock'', then the player should  exploit the adversary  by  consistently  playing ``Paper'' to achieve larger rewards. %Thus, a more robust learning algorithm that works well in both pessimistic and optimistic scenario with limited information of the game is desired.
No-regret algorithms \cite{cesa2006prediction,shalev2011online} prescribe a learning scheme in which a player is guaranteed to achieve minimal regret against the best fixed strategy in hindsight when facing an unknown adversary  (either rational or irrational).
Notably, if both players follow no-regret algorithms, then it is guaranteed that their time-average policies will converge to a NE in zero-sum games~\cite{blum2007learning}. However, the regret bound  of popular no-regret algorithms~\cite{freund1999adaptive,auer2002nonstochastic} usually depend on the game size; for example, Multiplicative Weight Update (MWU)~\cite{freund1999adaptive} has a regret bound of $\mathcal{O}(\sqrt{T \log(n))}$ and  EXP3~\cite{auer2002nonstochastic} in bandit setting has a regret of $\mathcal{O}(\sqrt{T n \log(n)})$, where $n$ is number of pure strategies (i.e., experts). As a result, directly applying no-regret algorithms, though rational, cannot solve large-scale  zero-sum games.  

In this paper, we seek for a scalable solution to two-player zero-sum normal-form games where the game size (i.e., the number of pure strategies) is prohibitively huge. 
Our main analytical tool is the no-regret analysis in online learning \cite{shalev2011online}. Specifically, 
by combining the no-regret analysis \cite{freund1999adaptive} with DO methods, we  propose \emph{Online Double Oracle (ODO)} algorithm. 
ODO inherits the key benefits from both sides. It is the first DO method that 
%is provably convergent to the NE (i.e.,  $1^\text{st}$ challenge above), and it 
enjoys the no-regret property and can exploit unknown types of adversary during the game play.  Importantly,  % with no coordination between players in order to solve sub-games %(i.e., $1^\text{st}$ challenge above) %\yaodong{both ODO and DO does not need} \lecong{now?}. 
%Furthermore, 
%unlike many no-regret algorithms \cite{freund1999adaptive,auer2002nonstochastic} whose regret bound depends on the game size, thus  inapplicable on  large games, 
our ODO method achieves a regret of $\mathcal{O}(\sqrt{T k \log(k)})$  where $k$ only depends on the support size of the NE rather than the game size. % (i.e., $3^{rd}$ challenge above).
We test our algorithm on tens of games including random matrix games and real-world games of different sizes \cite{czarnecki2020real},  including Kuhn Poker and Leduc Poker. Results show that in almost all games, ODO outperforms both vanilla DO and strong PSRO   \cite{lanctot2017unified,mcaleer2020pipeline} and online learning baselines  \cite{freund1999adaptive} in terms of exploitability (i.e., distance to an NE) and average payoffs against different adversaries. 
\vspace{-5pt}
\section{Related Work}
\vspace{-5pt}
The novelty of our ODO methods contributes to both  game theory domain and online learning domain. We make the list of existing game solvers for comparisons in Table \ref{tbl:gamesolver}. 
%Double Oracle part 
%Solving large two player zero-sum game with best oracle and its application \cite{hellerstein2019solving} 
%\yaodong{please cite \cite{gilpin2008first}, \cite{adam2020double}}

%\lecong{related work If the adversary is oblivious(i.e., the place of sensors will be the same in each round regardless the moves of the robust), then the robust can apply no-regret algorithm in Online Markov Decision Processes to achieve a good average performance. However, if the adversary is an adaptive learner who can spot the usual ways the robot uses and place the sensors in that ways to detect to robot, then the current state-of-art no-regret algorithms can not be applied in this non-oblivious environment}
%Algorithms such as fictitious play \cite{browniterative} and double oracle (DO) \cite{mcmahan2003planning} are the direct embodiments of this concept. 
% \cite{lanctot2019openspiel}.  \dm{Mention generalised weakened FP here?}

% No-regret algorithms have been widely used due to its guaranteed performance against the best-fixed strategy in  hindsight 
%\cite{freund1999adaptive,shalev2011online,cesa2006prediction}. 
% In solving normal-form games, if each player follows a no-regret algorithm, then the learning dynamics would converge to a Hannan set of strategies \cite{cesa2006prediction} \yaodong{either explain hannan or say Nash}. 

Approximating NE has been extensively studied in  game theory and multi-agent learning domains \cite{yang2020overview}. 
FP~\cite{browniterative} and generalised FP~\cite{leslie2006generalised} are classic solutions where each player adopts a policy that best responds to the time-average policy of the adversary. 
Although FP is provably convergent to NE in zero-sum games, they are prohibited from solving large games since it suffers from the curse of dimensionality due to the need to iterate all pure strategies in each iteration; and, the  convergence rate depends exponentially on the game size~\cite{brandt2010rate}.  %\yaodong{this claim may be wrong \lecong{in general cases, it is correct. we have different covergence for FP only in specific games}}.  %\yaodong{can we say this}.   
On solving large-scale zero-sum games, DO \cite{mcmahan2003planning, mcaleer2021xdo} and PSRO methods \cite{lanctot2017unified, mcaleer2020pipeline,nieves2021modelling} have shown remarkable empirical successes. For example, a distributed implementation of PSRO can  handle games of size $10^{50}$ \cite{mcaleer2020pipeline}. 
%Yet, there is still lack of theoretical insight; the worst-case scenario of applying DO/PSRO method is to solve the original game.  
Yet, both FP and DO methods offer no knowledge about how to exploit the adversary \cite{hart2001general}, thus regarded as not \emph{rational} \cite{bowling2001rational}.  Modern solutions that are rational such as CFR methods  \cite{zinkevich2007regret,lanctot2009monte} are designed  for extensive-form games only. 

%can also lead to Nash equilibrium convergence in this setting, but without regret performance guarantee against adversary \yaodong{not sure whether this claim is correct, need a reference}. 
Algorithms with no-(external) regret property can achieve guaranteed performance against the best-fixed strategy in  hindsight \cite{shalev2011online,cesa2006prediction}, thus they are commonly applied to tackle adversarial environments. 
%\yaodong{what does good strategy mean \lecong{strategy with high average payoff}}. 
%\yaodong{all algos have the same regret? and missing citations. and where is the exp3, hedge etc? \lecong{they will all depend on the number of experts, we have add exp-3 in}}
However, conventional no-regret algorithms such as 
%\yaodong{add citation one by one} 
%\yaodong{put your new online learning baseline here. OMD etc} \lecong{we may not have OMD in our experiment as it takes too long to run for larger game. Instead, we will have MWU and DO/w MWU as baseline for average performance. If we can manage to run OMD, then we can include it later}
Follow the Regularised Leader~\cite{shalev2011online}, Multiplicative Weight Update (MWU)~\cite{freund1999adaptive} or EXP-3~\cite{auer2002nonstochastic}  have the regret bound that is based on the number of pure strategies (i.e., experts). %Specifically, MWU has a regret bound $R_T=\mathcal{O}(\sqrt{T \log(n)})$, and EXP-3 has a regret bound of $R_T=\mathcal{O}(\sqrt{T n\log(n)})$. 
%\yaodong{why log is bad?\lecong{it is not bad but depend on the size of strategy set n}}
Moreover, these algorithms consider the full strategy set during the update, which deter their applications on large size games.
In this paper, we leverage the advantages of both   DO methods and no-regret learning algorithms to propose the ODO method. Our ODO enjoys the benefits of both being applicable to solving  large games and being able to exploit opponents. 
\vspace{-5pt}
\section{Notations \& Preliminaries}
\vspace{-5pt} 
%\yaodong{pls don;t start from a MDP, start from a normal-form game. and then you can say my solution formulat solving MDPS into solvin NFGs}\lecong{changed}
A  two-player zero-sum normal-form game is often described by a payoff matrix $\mA$ of size $n \times m$. 
%or even open-ended \Haitham{What does open ended mean?}
The rows and columns of $\mA$ are the pure strategies of the row and the column players, respectively.
We consider $n$ and $m$ to be prohibitively large numbers. % (e.g., $3^{936}$ in Leduc Poker\footnote{Leduc Poker can be efficiently represented by an extensive-form game. In this work, we used the normal-form representation to show the potential of our methods in solving large games. } \cite{southey2012bayes}).  
%We consider a repeated two-player zero-sum game. It can be described by a matrix $\mA$, where $\mA$ is an $n \times m$ non-zero matrix with entries in $[0,1]$ ($n$ and $m$ can be large or even open-ended).The rows and columns of $\mA$ represent the pure strategies of the row and column players respectively.
We denote the set of pure strategies for the row player as ${\Pi}:=\{{\va^1}, {\va^2},\dots {\va^n}\}$, and ${C}:=\{\vc^1, \vc^2, \dots, \vc^m\}$ for the column player. 
The set of row-player mixed strategies  is written as $\Delta_{\Pi}:= \big\{{\vpi}|{\vpi}= \sum_{i=1}^n x_i {\va^i},\; \sum_{i=1}^n x_i=1, x_i \geq 0, \forall i \in [n] \big\}$, and for the column player, it is $\Delta_{C}:=\{\vc|\vc= \sum_{i=1}^m y_i \vc^i, \sum_{i=1}^n y_i=1, y_i \geq 0, \forall i \in [m]\}$.  %denotes the  mixed-strategy set for the column player. 
The support of a mixed strategy is written as $\text{supp}(\vpi):=\{\va^i \in \Pi | x_i \neq 0 \}$, with its size  $|\text{supp}(\vpi)|$. 
%\yaodong{define support and effectvei strategy here?} \lecong{can not define effective strategy set without the OSO algorithm}

We consider $\mA_{i, j} \in  [0,1]$ representing the (normalised) loss of the row player when playing against the column player. 
At the $t$-th round, 
%\Haitham{Where does this round come from. Maybe you want to say the play against each other in rounds or something like that. \lecong{it is a repeated game so I think it is clear to say about round of the game}}
the payoff for the joint-strategy profile $(\vpi_t \in \Delta_{\Pi}, {\vc}_t \in \Delta_{C})$ is written as $(-\vpi_t^{\top}\mA \vc_t, \vpi_t^{\top}\mA \vc_t)$. 
%the row player chooses a mixed strategy $\vpi_t \in \Delta_{\Pi}$ (resp. ${\vc}_t \in \Delta_{C}$), then the row player's payoff is $-\vpi_t^{\top}\mA \vc_t$ while the column player's payoff is $\vpi_t^{\top}\mA \vc_t$. 
The row player 's goal is to minimise  the quantity $\vpi_t^{\top}\mA \vc_t$. 
 In this paper, we consider the online setting in which players do not know the matrix $\mA$ and adversary's policy, but rather receive a loss value after the strategy is played:  at timestep $t+1$, the row player observes $\vl_t=\mA \vc_t$ given  by the environment, and then it plays a new strategy $\vpi_{t+1}$.    
% Finally, we denote $\langle \cdot, \cdot \rangle$ as the inner product of two vectors. 
 
\textbf{Nash Equilibrium.} NE of two-player zero-sum games can be defined by  the minimax theorem \cite{neumann1928theorie}: 
\begin{equation} \label{minimax theorem}
    \min_{\vpi \in \Delta_{\Pi} } \max_{\vc \in \Delta_{C} } \vpi^{\top}\mA \vc=  \max_{\vc \in \Delta_{C} } \min_{\vpi \in \Delta_{\Pi} } \vpi^{\top}\mA \vc =v,
\end{equation}
for some $v$ $\in \mathbb{R}$. The $(\vpi^*, \vc^*)$ that satisfies Equation (\ref{minimax theorem})  %\Haitham{Capitalise equation = Equation} 
is the NE of the game. 
In general, one can apply LP method to solve the NE in small games \cite{morgenstern1953theory}. However, when $n$ and $m$ grows large, the time complexity is not affordable. 
%\Haitham{Tell me why you need a more general notion saying something like, this is hard to compute and then we need approximation or whatever the reason is. But this seems arbitrary from Nash to epsilon.} 
A more general solution concept  is the $\epsilon$-Nash equilibrium, written as  

\textbf{$\epsilon$-Nash Equilibrium.} For  $\epsilon >0$,  we call a joint strategy $(\vpi, \vc) \in \Delta_{\Pi} \times \Delta_{C}$ an $\epsilon$-NE if it satisfies 
\begin{equation}
    \max_{\vc \in \Delta_C} \vpi^{\top}\mA \vc -\epsilon \leq \vpi^{\top}\mA \vc \leq \min_{\vpi \in \Delta_{{\Pi}}}  \vpi^{\top}\mA \vc +\epsilon.
    \label{eq:ene}
\end{equation}
%\Haitham{All equations should be punctuated, either a . or a comma}
%When $\epsilon$ goes to 0, $\epsilon$-NE becomes the exact NE.
%In symmetric games (i.e., $n = m$), 
%%\yaodong{undefined}, 
%the support size of NE can be as large as the game's size. In asymmetric games where  $n \gg m$, we can bound the support size of NE by the following lemma. 
Many NE approximation methods (e.g., DO \cite{mcmahan2003planning} / PSRO \cite{lanctot2017unified}) are developed based on the  assumption that the support size of NE is small. Formally, it is written as follows. 
\begin{assumption}[Small Support Size of NEs]\label{small support size assumption}
Let $|\cdot|$ denote the cardinality of a set, $(\vpi^*, \vc^*)$ be a NE of the game $\mA_{n \times m}$, we assume the support size of $(\vpi^*, \vc^*)$ is smaller than the game size:. 
\begin{equation}
\max\big(|\operatorname{supp}(\vpi^*)|, |\operatorname{supp}(\vc^*)|\big) <  \min(n,m).
\end{equation}
\label{keyassump}
\vspace{-15pt}
\end{assumption}
Assumption \ref{small support size assumption}  holds in many settings. 
In symmetric games with random entries  \cite[Theorem 2.8]{jonasson2004optimal}, it has been proved that the expected support size of NE will be $(\frac{1}{2}+\mathcal{O}(1))n$ where $n$ is the game size; this means the support size of a NE strategy is only half of the game size. 
In asymmetric games (e.g., $n\gg m$), we provide the following lemma under which Assumption \ref{small support size assumption} also holds. 
\begin{lemma}\label{lemma bound on k}
In asymmetric games $\mA_{n \times m}, n \gg m$, if the 
%Suppose the matrix $\mA_{n \times m}$ has a unique %\yaodong{do we really need uniquness?} \lecong{we will need uniqueness for this lemma} 
NE $(\vpi^*, \vc^*)$ is unique, then the support size of the NE  will follow $|\operatorname{supp}(\vpi^*)| =  |\operatorname{supp}(\vc^*)| \leq m 
$. (The full proof is in the Appendix {\color{blue}{A.1}}.)
% \begin{equation*}
% \max(\big|\operatorname{supp}(\vpi^*)\big|, \big|\operatorname{supp}(\vc^*)\big|) \leq \min(n,m),
% \end{equation*}
\end{lemma}
Notably,  the premise of a unique NE  in Lemma \ref{lemma bound on k} is a generic property  for zero-sum games. On the set of zero-sum normal-form games, the set of zero-sum games with non-unique equilibrium has Lebesgue measure zero  \cite{van1991stability}.
Empirically, we also show that the Assumption \ref{small support size assumption} holds on tens of real-world zero-sum games \cite{czarnecki2020real} (see Table 2 in Appendix {\color{blue}{C}}) and randomly generated games (see Figure \ref{size of k figure}).  
Later in Section \ref{sec:large-k}, we develop  solutions when Assumption \ref{small support size assumption} is violated.  
%We also introduce a result about the support size of a NE, which underpins  theoretical guarantees of later sections. 
%\begin{lemma}\label{lemma bound on k}
%Suppose the matrix $\mA_{n \times m}$ has a unique \yaodong{do we really need uniquness?} \lecong{we will need uniqueness for this lemma} Nash equilibrium $(\vpi^*, \vc^*)$, then the support of the Nash equilibrium of the row player will follow:
%\begin{equation*}
%\max(\big|\operatorname{supp}(\vpi^*)\big|, \big|\operatorname{supp}(\vc^*)\big|) \leq \min(n,m),
%\end{equation*}
%where $|\cdot|$ denotes the cardinality of a set. 
%\end{lemma}
%The proof is given in the Appendix A.1. Notably,  the unique \yaodong{unique?} NE premise in Lemma \ref{lemma bound on k} is a generic condition. It has been shown that within the set of discrete zero-sum games, the set of zero-sum games with non-unique equilibrium has Lebesgue measure zero \cite{van1991stability}. %Therefore, if the entries of payoff matrix $\mA$ are sampled independently from some continuous distribution, then the game $\mA$ has a unique NE with probability one.

%\yaodong{where this result is used. give a hint why we need this lamma}
\vspace{-5pt}
\subsection{Double Oracle Method}\label{section: Double Oracle}
\vspace{-5pt}
%\yaodong{horrible section, all notaitons are not defined. please rewrite}
%\yaodong{don't repeat what the pseduco code says, but rather describe the process. check my diveristy paper}
\begin{wrapfigure}{R}{0.6\textwidth}
\vspace{-3pt}
\begin{algorithm}[H]
\caption{Double Oracle Algorithm \cite{mcmahan2003planning}}
\label{Double Oracle algorithm}
\begin{algorithmic}[1]
\STATE {\bfseries Input:} Full  strategy set $\Pi, C$ %of players
\STATE  Initialise sets of strategies $\Pi_0, C_0$ 
\FOR{\text{$t=1$ to $\infty$}}
\IF{ $\Pi_{t}\neq \Pi_{t-1}$ or $C_t \neq C_{t-1}$}
\STATE Solve the NE of the sub-game  $\mG_t$: \\
$(\vpi^*_t, \vc^*_t)=\arg\min_{\vpi \in \Delta_{\Pi_t} } \arg\max_{\vc \in \Delta_{C_t} } \vpi^{\top}\mA \vc$
%and obtain the Nash equilibrium $((\vpi^*_t, \vc^*_t))$ \yaodong{solve Nash please use math to describe}
\STATE Find the best response $\va_{t+1}$ and $\vc_{t+1}$ to $(\vpi^*_t, \vc^*_t)$:  \\
\ \ \ \ \ \ \ \ \ \ \ \ \ $\va_{t+1}=\arg\min_{\va \in \Pi} \va^{\top} \mA\vc^*_t $ \\
\ \ \ \ \ \ \ \ \ \ \ \ \ $\vc_{t+1}=\arg\max_{\vc \in C} {\vpi^*_t}^{\top} \mA \vc$
%\yaodong{where is this BR defined, also for all texts in formula, pleas use $\operatorname{BR}$}
\STATE  $\Pi_{t+1}= \Pi_t \cup \{\va_{t+1}\}, C_{t+1}= C_t \cup \{\vc_{t+1}\}$
\ELSIF{$\Pi_{t} = \Pi_{t-1}$ and $C_t = C_{t-1}$} 
\STATE Terminate
\ENDIF
\ENDFOR
\end{algorithmic}
\end{algorithm}
\vspace{-20pt}
  \end{wrapfigure}
The pseudocode of DO algorithm ~\cite{mcmahan2003planning} is listed in Algorithm \ref{Double Oracle algorithm}. 
The DO method approximates NE in large-size zero-sum games by iteratively creating and solving a series of sub-games (i.e.,  game with a restricted set of pure strategies). 
Specifically, at time step $t$, the DO learner solves the NE of a sub-game $\mG_t$.  
Since the sets of pure strategies of the sub-game $\mG_t = (\Pi_t, C_t)$ are often much smaller than the original game, the NE of sub-game $\mG_t$ can be easily solved in line $5$.     
Based on the NE of the sub-game: $(\vpi^*_t, \vc^*_t)$, each player finds a best response to NE (line $6$), and expand their strategy set (line $7$). PSRO methods \cite{lanctot2017unified,mcaleer2020pipeline} are variants of DO methods in which RL methods are adopted to approximate the best response strategy. 
For games where Assumption \ref{keyassump} does not hold, DO methods restore to solve the original game and have no advantages over LP solutions.  
%In a game with finite pure strategy set $\Pi, C$ and known matrix $\mA$, the DO algorithm guarantees convergence to the Nash equilibrium of the game.  

Although DO can solve large-scale zero-sum games, it requires to \emph{coordinate} both players to consistently find the NE (see line 5 in Algorithm \ref{Double Oracle algorithm}); this  renders an disadvantage of DO that when applied in real-world games, it cannot exploit the opponent who can play any non-stationary strategy  (see the example of RPS in Introduction). 
Our ODO solution address this problem by combining DO with tools in online learning. 

%DO is designed to find the NE of the game, and it assumes that it has full control over the both players. In situations where there is no coordination between players, then DO method will not guarantee the convergence to an NE. Further, DO method is not rational in the sense that it does not provide a learning schema that can exploit the adversary (see the example of RPS in Introduction).
%Although DO is provably convergent to NE in finite zero-sum games, 
%in the worst case, it still  recovers to solve the original  game $\mA$. Moreover, the DO algorithm requires the players to know the game matrix $\mA$ to find the NE of the subgame (line $4$) (i.e., without the knowledge of the game matrix, DO requires player to collobarate to find the NE) \yaodong{problematic claim!!} \lecong{now}, and it is not rational in the sense that it does not provide a learning schema that can exploit the adversary (see the example of RPS in Introduction). 
%All of these challenges prevent it from being applied on adversarial setting. 

%\Haitham{Another arbitrary transition, why did u decide to go to online learning what's the pt of that. We were in nash suddenly jumped. Walk me by the hand tell me why u now talk about OL}
\vspace{-5pt}
\subsection{Online Learning}
\vspace{-5pt}
Solving the NE in large-scale games is demanding. An alternative approach is to consider learning-based methods. % that do not necessarily need access to the game matrix $\mA$. 
We hope that by playing the same game repeatedly,  a learning algorithm could  approximate the  NE asymptotically.  
A common metric to quantify the performance of a learning  algorithm is to compare its cumulated payoff with the best fixed strategy in hindsight, which is also called  \emph{(external) regret}. %Formally, we have the following definition.   %\zheng{strategy? to be consistent in language} 
\begin{definition}[No-Regret Algorithms]
Let $\vc_1, \vc_2, \dots$ be a sequence of mixed strategies played by the column player, an algorithm of the row player that generates a sequence of mixed strategies $\vpi_1, \vpi_2, \dots$ is called a no-regret algorithm if we have
\vspace{-5pt} %\yaodong{why it is not =0 \lecong{so when it is negative, the algorithm can outperform the best fixed strategy in the hindsight. the online learning community hardly considers this problem as currently there is no algorithm like it. But taking the advice from David and others to make it concrete.}}
    \begin{align}
    \lim_{T \to \infty} \frac{R_T}{T} =  0, \ \ \ \  R_T= \max_{\vpi \in \Delta_{\Pi}} \sum_{t=1}^T \left( \vpi_t^{\top}\mA \vc_t-\vpi^{\top} \mA \vc_t\right).\end{align}
% where the regret is written as 
%     \begin{align}
%  \;\;& R_T= \max_{\vpi \in \Delta_{\Pi}} \sum_{t=1}^T \left( \vpi_t^{\top}\mA \vc_t-\vpi^{\top} \mA \vc_t\right).
%     \end{align}
    \vspace{-15pt}
% \lecong{by changing the definition to $\lim ... \leq 0$, we solve one of the comment of the reviewer that the algorithm can outperform the best fixed strategy, however, this definition is not standard. what do you think?}\zheng{I think the confusion comes from that reviewers (not only the one who made this comment) will easily think that piAc is the payoff of the row player but not the loss. Therefore, the Def.2 becomes confusing as when you have large regret it actually means you achieved higher payoff. See the comment 6 just below this comment in the reviews, the reviewer is indeed confused by how you call matrix A as payoff matrix but row player needs to minimize it.}{\lecong{changed the payoff matrix A to matrix A, with is fine. Also I think it is more accurate to use $\leq 0$ in the above definition }}
\end{definition}
%\yaodong{introduce MWU, 1) show the update rule, 2) convergence bound, 3) and mentioned when two no regret play, then NASH}
No-regret algorithms are of great interest in two-player zero-sum games since  if both players follow a no-regret algorithm (not necessarily the same one), then the average strategies of both players converge to a NE of the game~\cite{cesa2006prediction,blum2007learning}. For example, 
a well-known learning algorithm for games that has the no-regret property is the MWU algorithm \cite{freund1999adaptive}, which is described as %which slowly updates the strategy based on the previous loss in each iteration. The detail algorithm can be described as follow: 
\begin{definition}[Multiplicative Weights Update]\label{MWU definition} Let $\vc_1, \vc_2, ...$ be a sequence of mixed strategies played by the column player. The row player is said to follow the MWU if  $\vpi_{t+1}$ is updated as follows 
%\yaodong{can you put the two lines into one equation. i don't see why we split, and also, what is $\pi_0=?$}:
\vspace{-3pt}
\begin{equation}\label{eq: MWU update}
    \begin{aligned}
    \vpi_{t+1}(i)=\vpi_t(i)  \frac{\exp({-\mu_t {\va^i}^{\top} \mA \vc_t})}{\sum_{i=1}^n \vpi_t(i)\exp({-\mu_t {\va^i}^{\top} \mA \vc_t})}, \ \  \forall i \in [n]
    \end{aligned}
\end{equation}
where $\mu_t >0$ is a parameter, $\vpi_0=[1/n,\dots,1/n]$ and $n$ is the number of pure strategies (a.k.a. experts).   The regret of MWU can be bounded by $ R_T=\max_{\vpi \in \Delta_{\Pi}}R_T(\vpi) \leq \sqrt{\frac{T\log(n)}{2}}$. 
%With a proper setup of $\mu_t$ in MWU, the regret $R_T$ can be bounded by~\cite{freund1999adaptive}:
% \begin{equation}\label{MWU regret bound}
%     R_T=\max_{\vpi \in \Delta_{\Pi}}R_T(\vpi) \leq \sqrt{\frac{T\log(n)}{2}}, 
% \end{equation}
\vspace{-5pt}
\end{definition} 
Intuitively, the MWU method functions by putting larger weights on the experts who have lower losses in the long run. Thus, compared to the best-fixed strategy in hindsight (i.e., the expert with the lowest average loss), the MWU can achieve the no-regret property. However, since the MWU requires updating the whole pure strategy set in each iteration, it is not applicable on large-size games.
%Intuitively, at each iteration, the MWU updates ensure to increase the weights of pure strategies with low loss slowly, thus preventing the adversary from exploiting the process\yaodong{rewrite this sentence, not clear, suggest to mimic how other paper describe it}.
%\yaodong{we need more intution of MWU, why it works, 1-2 sentence}

%\begin{remark}
%We note that our main results, the Online Single Oracle~\ref{Online Single Oracle algorithm} and Online Double algorithm , can still apply in situations where the loss function $\vl_t$ takes different forms rather than linear loss function (i.e., the only requirement is that the no-regret algorithms can be applied to the effective strategy $\Pi_t$ and the strategy set is discrete). We consider the matrix form setting through the paper to better demonstrate the idea of our algorithm. %zheng{I feel this is unclear to readers. e.g. we never introduced what is meta game before this}\lecong{fixed}
%\end{remark}
%\yaodong{please add a parapgrah here with pseduo-code of what is double oracle}\lecong{added: maybe we need to move it to the Appendix} 

\vspace{-5pt}
\section{Online Double Oracle}
\vspace{-5pt}
%\Haitham{As we discussed this needs a story before the details \lecong{how about now?}}
In this section, we introduce the ODO algorithm.   
ODO enjoys the no-regret property. Compared to the DO method, ODO can  play strategically to exploit the opponent. 
Compared to the MWU algorithm, ODO can be applied on solving large zero-sum games. 
The following subsections are organised as follows. 
We first introduce the key building block of ODO, which is Online Single Oracle (OSO). % algorithm and derive its regret bound. 
ODO is the algorithm in which both players adopt the OSO method. We  derive ODO's convergence rate to NE, and also  analyse ODO's performance when players can only access \emph{approximate} best responses.  
Finally, we provide a robust extension of OSO considering cases when the size of NE support is unknown (i.e.,  Assumption \ref{keyassump} may not hold). %,  we provide an effective extension of OSO that maintains OSO's performance when Assumption \ref{keyassump} holds while guarantees to perform not worse than standard no-regret algorithms otherwise. \yaodong{change claim, still maintain no-regret proert is incorrect} 

\subsection{Online Single Oracle Algorithm}\label{Section Online Single Oracle}
%The development of ODO is based on the key assumption that the support size of a NE in zero-sum game is often smaller than the game size (see Lemma \ref{lemma bound on k}).  
%Here we introduce OSO. 
One can think of OSO as an online  counterpart of \emph{Single Oracle} in DO \cite{mcmahan2003planning}, but achieves no-regret property. 
A key advantage of OSO is that its regret bound 
 does not depend on a player's size of strategy set, 
 but rather   the size of so-called \emph{effective strategy set}, a quantity that is linearly dependent on  the size of  NE support. %, which is usually a small number  according to Assumption \ref{small support size assumption}.  
 
%We  introduce a new type of no-regret algorithm -- Online Single Oracle -- whose regret bound does not depend on the player's size of strategy set, but rather   the size of \emph{effective strategy set}, which relates to  the support size of the NE under Assumption \ref{small support size assumption}. 

\begin{wrapfigure}{R}{0.6\textwidth}
\vspace{-1pt}
\begin{algorithm}[H]
\caption{Online Single Oracle Algorithm}
\label{Online Single Oracle algorithm}
\begin{algorithmic}[1]
\STATE {\bfseries Input:} Player's pure strategy set $\Pi$ 
\STATE  Init. effective strategies set:  $\Pi_0=\Pi_1= \{\va^j\}, \va^j \in \Pi$
\FOR{\text{$t=1$ to T}}
\IF{ $\Pi_{t}=\Pi_{t-1}$} 
\STATE 
%\yaodong{we can quote the equation from where MWU is introduced. no need to repeat \lecong{did it the last time, Jun suggested to put it here to make the alg self-contained}}
Compute $\vpi_t$ by the MWU  in Equation (\ref{eq: MWU update}) 
%\yaodong{save the space by referring to Eq. 4}
%$\vpi_{t}(i)=\vpi_{t-1}(i)  \frac{\exp(-\mu_{t-1} {\langle \va^i,l_{t-1} \rangle})}{\sum_{i=1}^n \vpi_{t-1}(i)\exp({-\mu_{t-1} {\langle \va^i,l_{t-1} \rangle}})}\;$ 
%\pi^{i} \in \Pi_t$
\ELSIF{$\Pi_{t}\neq \Pi_{t-1}$}
\STATE Start a new time window $T_{i+1}$ and \\ Reset $\vpi_t= \big[1/|\Pi_{t}|,\dots,1/|\Pi_{t}|\big], \ \  \bar{\vl}=\textbf{0}$ 
%\STATE Reset the MWU update in Equation (\ref{eq: MWU update}) with a new initial strategy $\vpi_t$
%with the new effective strategy set: $\Pi_t$ \yaodong{use math. what is reset? what is new effetive strategy set?}\lecong{will add the agurment about initial strategy in the explaination}
\ENDIF
\STATE Observe $\vl_t$  and update the average loss in  $T_i$:  %\yaodong{change "in the current time window" to the math}: 
$\bar{\vl}= \sum_{t \in T_i} \vl_t /{|T_i|}$
\STATE Calculate the best response:  %\yaodong{strange notation, so this a is a pi?\lecong{a represents pure strategy}}
  $\va_t=\arg\min_{\vpi \in \Pi} \langle \vpi,\Bar{\vl} \rangle$
\STATE Update the set of strategies: 
 $\Pi_{t+1}=\Pi_t \cup \{\va_t\}$
%\STATE Output the strategy $\vpi_t$ at round $t$ for the player
\ENDFOR
\STATE {\bfseries Output:} $\vpi_T$, $\Pi_T$
\end{algorithmic}
\end{algorithm}
\vspace{-10pt}
\end{wrapfigure}
In contrast to classical no-regret algorithms such as MWU \cite{freund1999adaptive} where the whole set of pure strategies   needs considering at each iteration, i.e., Equation (\ref{eq: MWU update}),  
%set has to be considered when developing 
%\Haitham{Updating?} 
%a new mixed strategy (i.e., Equation (\ref{eq: MWU update})),
%\yaodong{cite equation where you introduec MWU}
we propose OSO  that only  considers a \emph{subset} of the  whole strategy set. 
The key operation is that, at each round $t$, OSO only considers adding a new strategy  
%a new pure strategy will only be added to thesubset 
if it is the best response to the  average loss in a time window (defined later). %The time window will guarantee the OSO's performance while only considers a small subset of pure strategies. % of the current time window. 
%\yaodong{undefined} \lecong{we defined in detail later}. 
%\yaodong{explain why we need time window}
As such, OSO could save exploration cost and computation time  by ignoring the pure strategies that have never been the best response to any average losses $\bar{\vl}$ observed so far, but rather focusing on those effective strategies.   This can  benefit the learner especially in solving extremely large games. %\yaodong{when does OSO stop? and waht is the output of OSO} %, for example,  in the later stage of training. 
%In the rest of this section we will formally introduce the OSO algorithm and provide a regret bound for it.

%The main idea in OSO algorithm \ref{Online Single Oracle algorithm} is that at each round $t$, the player only considers a certain smart pure ``strategies''\zheng{why do you quote strategies here? you didn't define smart in this paragraph.} (i.e., the usual no-regret algorithm will consider the whole set of pure strategies and thus in situations(e.g., online MDPs) where the set of strategies is super large, it is not efficient). After each round, based on the performance of the strategies, a new strategy can be added to the set and continue the process. Intuitively, this method will ignore  ``unnecessary'' strategies from the update and thus the method is much more efficient. The algorithm can be described as follow:

%\yaodong{Add one-two sentence of high-level summary of how you build OSO here. E.g., how MWU is added to the DO process. You can build the link by setting the connection between sec 3.1, 3.2 \lecong{should be in the ODO section where we have the meta game}}
The pseudocode of OSO is listed in Algorithm \ref{Online Single Oracle algorithm}. 
We initialise the OSO algorithm with a random subset $\Pi_0$ from the original  strategy set $\Pi$. Without losing generality, we assume that $\Pi_0$ starts from only one pure strategy (line $2$).  
We call subset $\Pi_t$ the \textbf{effective strategy set} at time $t$, and define the period of consecutive iterations as one {time window} $T_i$ in which the effective strategy set stays fixed, i.e., $ T_i:=\big\{t \ \big| \  |\Pi_t|=i\big\} $. %, and $T$ is indexed by the cardinality of the strategy set, that is,  $  T_i=\big\{t \ \big| \  |\Pi_t|=i\big\}  $. 
% \begin{equation*}
%   T_i=\big\{t \ \big| \  |\Pi_t|=i\big\}.  
% \end{equation*}
%\yaodong{be more specific, i still don;t see how $T_i$ is defined in math} %\yaodong{define by math pls} 
At iteration $t$, we update $\vpi_t$ (line $5$)  where only  the effect strategy set $\Pi_t$ (rather than whole set $\Pi$) is considered;  and  the best response is computed against  the average loss $\bar{\vl}$ within the current  $T_i$  (line $9$). 
 Adding a new best response that is not in the existing  effective strategy set will start a new time window (line $7$). 
%\[\bar{\vl}=\frac{1}{|T_i|} \sum_{t \in T_i} \vl_t.\]
%\yaodong{math math math, and linke to the pseducode}. 
Notably, despite the design of effective strategy sets, the exact best response of OSO  in line $10$ still needs considering the whole strategy set $\Pi$; we will  relax it through approximating best responses in Section \ref{sec:abr}. % \yaodong{correct?}. \lecong{yes}  %, and the accumulated loss $\bar{\vl}$ within one time window will be reset to zero if a new time window starts.  

We now present the regret bound of OSO in Algorithm \ref{Online Single Oracle algorithm} as follows. 
\begin{theorem}[Regret Bound of OSO] \label{regret of online oracle with stationary distribution}
Let $\vl_1,\vl_2,\dots, \vl_T$ be a sequence of loss vectors played by an adversary, and $\langle \cdot , \cdot \rangle$ be the dot product,  OSO in Algorithm \ref{Online Single Oracle algorithm} is a no-regret algorithm with 
\[\frac{1}{T}\Big(\sum_{t=1}^T \big\langle \vpi_t,\vl_t \big\rangle - \min_{\vpi \in \Pi} \sum_{t=1}^T \big\langle \vpi , \vl_t \big\rangle \Big) \leq \frac{\sqrt{k \log(k)}}{\sqrt{2T}},\]
where $k = |\Pi_T|$ is the size of effective strategy set in the final time window.

(We provide the full proof in the Appendix {\color{blue}{A.2}}.)

%\yaodong{how this effective strategy size is related to support size}\lecong{explain in paragraph following paragraph}
%\yaodong{i am not quite comfortable reading here, where i am still unclear about what is effective strategy set, and time window. can we defein them somewhere using math} \lecong{changed it. How about now?}
%\yaodong{this k is quite imporatnat, need an independent definition maybe. we cannot define such an important notion by words\lecong{we do not have a mathematical formulation for k rather than the size of the final effective strategy}}
\end{theorem}
\begin{proof}(of sketch) 
OSO is designed in such a way  that the best response $\va_t \in \Pi$ to the average loss  must stay in the  effective strategy set $\Pi_t$; otherwise, a new time window would start. %(i.e., $\va \in \Pi_t \subset \Pi$). 
%\yaodong{don;t understand this sentence}. 
Thus, the best strategy in $\Pi_t$ must be as good as the best  strategy in the whole  set $\Pi$.  Therefore, we can use MWU to bound each time window's regret, resulting in the total regret  across all timesteps. % The  proof is in  Appendix {\color{blue}{A.2}}. 
%\yaodong{where? need label}. 
\end{proof}
%We provide the full proof in the Appendix. \Haitham{Add these in begin and end proof enviornment and give a proof sketch. A paragraph with couple of line summarising main steps needed for the proof.\lecong{agreed. changed}}
%\lecong{this may not fit well as the new algorithm needs to compute best response every round as well, but the difference is it does not add it to the effective strategy set. We can use the augument about the expensive best response in the ODO, where we know that the regret will grow slowly, thus instead of checking the best response every round, we only needs to check it after several round, saving the cost. We can not argue the same in the case of OSO here though as we do not know the regret will grow slowly or not}\lecong{have fixed it, check online double oracle for more detail}
%In many economics problems, following a new strategy (e.g., plan) is costly (e.g., due to the construction cost to follow the plan)~\cite{rass2017cost}. Thus, a good algorithm in these situations require to take the number of pure strategies into consideration while maintain the no-regret property. 
%Computing the best response can be computationally expensive in many real-world tasks. Particularly, if the game is large, optimisation subroutines such as model-free reinforcement learning algorithms or gradient-descent methods have to be used, and it could take hours or  days to obtain one single best response \cite{vinyals2019grandmaster}.
%\lecong{this}

% \yaodong{le cong, can you wrap this linear dependency in a remark, finish it within page 5}\lecong{now}
\begin{remark}[The Size of Effective Strategy Set $k$] \label{small effective strategy set remark}
Similar to the original DO, in the  worst-case scenario, OSO has to list all pure strategies, i.e., $k = |\Pi|$. % restores the full size. 
However, 
we believe $k \ll |\Pi|$ holds in many cases. 
Intuitively, since OSO is a no-regret algorithm, should the adversary follow another no-regret algorithm, the adversary's average strategy would converge to the NE. Thus, learner's  effective strategy set with respect to the average loss will include all the pure strategies in the support of learner's NE, which, under Assumption~\ref{small support size assumption}, is  a far  smaller number compared to the game size. %, the size of effective strategy set--the $k$ %
%--will be much smaller than the size of pure strategy set $\Pi$. 
%In situations where there is a unique NE, the best response set with respect to the adversary's NE will have the same size as the player's support of the NE (i.e., $|S|= \operatorname{supp}(\vc^*)$ where $S=\{\va| \va \in \Pi, \va=\arg\max_{\va \in \Pi} \va ^{\top} \mA \vc^*\}$)~\cite{bohnenblust1950solutions}.
%As such, $k$ is usually linearly dependent on the  size of NE support. 
%Therefore, the total number of best response strategies added the the effective strategy set -- the $k$ -- will depend linearly on the support size of the NE. 
Later in Figure \ref{size of k figure}, we also 
provide empirical evidence to  support the claim of $k \ll |\Pi|$. 
We believe theoretically upper-bounding  $k$ is an important, yet challenging, future work for  both DO/PSRO and ODO series of methods. 
%show the evidence that $k$ depends linearly on the support size of the NE when the adversary follows a no-regret algorithm. 
%\label{remark:sizek}
\end{remark}

\textbf{OSO with Less-Frequent Best Response.}
Obtaining a best response strategy can be computationally expensive. It could take hours and days to obtain one single best response  \cite{vinyals2019grandmaster}. 
 Yet, OSO in Algorithm~\ref{Online Single Oracle algorithm} considers  adding a new best-response strategy  at  every iteration. 
A practical solution is to consider  adding  
%The OSO Algorithm~\ref{Online Single Oracle algorithm} \Haitham{Captilise Algorithm. Also sentence needs rewriting} will not only reduce the regret bound in the case the effective strategy set $k$ is small, it also reduces the hardness of computing no-regret algorithm when the number of pure strategy is large (e.g., Multiplicative Weight Update or Follow the Regularized Leader).
%However, by following the Algorithm~\ref{Online Single Oracle algorithm}, the player may add a new strategy $\va$ to the effective strategy set $\Pi_t$ too quickly (i.e., in every round), thus risking the chance of adding a poor strategy. Therefore, in situations where playing new strategy is expensive (e.g., due to the construction and exploration cost to play a new strategy or the cost to discover a new best response), the algorithm needs to be more cautious about playing a new strategy. Following this motivation, we suggest a new algorithm that 
a new strategy  when the regret in the current time window exceeds a predefined threshold $\alpha$. 
Specifically, if we denote $|\bar{T}_i|:= \sum_{h=1}^{i-1} |T_h|$ as the starting time  of the time window $T_i$, and write the threshold at  $T_i$ as 
   $ \alpha_{t-|\bar{T}_i|}^i $
where $t-|\bar{T}_i|$ denotes the relative position of round $t$ in the time window $T_i$. 
We can make OSO add a new strategy only when the following equation is satisfied: 
%resulting  improvement exceeds the threshold of $ \alpha_{t-|\bar{T}_i|}^i $:
%To summarise, we introduce a  variation of OSO  that only adds a new strategy in the effective strategy set $\Pi_t$  if the resulting  improvement exceeds the threshold of $ \alpha_{t-|\bar{T}_i|}^i $:   
%\yaodong{use indepedent line of math equation to expalin your last sentence}
%\yaodong{what is $|\bar{T}_i|;i$?}
%\yaodong{explain this sentecce using math, what is the regret in the current time wondow, what is the threshold. and also, what does $|\bar{T}_i|;i$ this notaiton mean????}:
\vspace{-7pt}
\begin{equation}
\min_{\vpi \in \Pi_t} \Big\langle \vpi, \sum_{j= |\bar{T}_i|}^t \vl_j \Big\rangle - \min_{\vpi \in \Pi} \Big\langle \vpi, \sum_{j=|\bar{T}_i|}^t \vl_j \Big\rangle \geq \alpha_{t-|\bar{T}_i|}^i. 
\label{eq:oasas}
\end{equation}
%\vspace{-5pt}
Note that the larger the threshold $\alpha$, the slower the OSO algorithm  adds a new strategy into $\Pi_t$. However, choosing a large $\alpha$ will prevent the learner from acquiring the actual best response, thus increasing the total regret $R_T$ by $\alpha$.  In order to maintain the no-regret property, the $\alpha$ needs to satisfy
%\yaodong{this is a sentence out of nowhere, what is the reason?}:
\vspace{-5pt}
\begin{equation}
\lim_{T \to \infty} \frac{\sum_{i=1}^k \alpha_{T_i}^{i}}{T}=0. 
\label{eq:ti}
\end{equation}
One  choice of $\alpha$ that satisfies Equation (\ref{eq:ti}) can be   $\alpha_{t-|\bar{T}_i|}^i=\sqrt{t-|\bar{T}_i|}$. 
%\yaodong{the below equation does not seem to be correct \lecong{why is that? }}:
% \begin{equation}
% 	\alpha_{t-|\bar{T}_i|}^i=\sqrt{t-|\bar{T}_i|}.
% \end{equation}
We list the pseudo-code of the OSO under Equation (\ref{eq:oasas}) and derive its  regret bound of $\mathcal{O}\big(\sqrt{k \log(k)T}\big)$  in Appendix {\color{blue}{A.3}}.  
%\yaodong{can we add the bound here}\lecong{added} \yaodong{need a theorem?}
%In the later section, we also show that such a variation can help reduce the cost of computing best response in self-pl
%In ODO algorithm, OSOSAS can also apply to reduce the cost of computing best response. See Section \ref{Online Double Oracle} for more details. 
%\yaodong{proof of what? and where to find. we cannot in general say, everything is in the appenfix, find yourself. be very specific, where to find} and algorithm can be found in the Appendix\lecong{need to include the proof about regret bound in}. Note that if we choose a small positive $\alpha$ then the original OSO algorithm is recovered \yaodong{this sentence is out of nowhere, why it is the case? is it based on some previous resulst. also, what does ``a small positive" mean? $0.01, 0.0001?$}. 
%\zheng{any justification of these setting? do we have results for this? if not we would better to only include the intuitive discussion here and move the implementation details to the appendix. \lecong{We will have a regret bound for this one. This I think is the same idea that Nicholas uses in the experiment}}

%\yaodong{we need to mention why we do so, it is because, self-play OSO is convergeing to Nash}
\subsection{Online Double Oracle}\label{Online Double Oracle}
\begin{wrapfigure}{R}{0.6\textwidth}
\vspace{-15pt}
\begin{algorithm}[H]
\caption{Online Double Oracle Algorithm}
%\lecong{should we include in the main tex}}
\label{Online Double Oracle algorithm}
\begin{algorithmic}[1]
\STATE {\bfseries Input:} Full pure strategy set $\Pi$, $C$ 
\STATE  Init. effective strategies set:  $\Pi_0=\Pi_1,C_0=C_1$
\FOR{\text{$t=1$ to T}}
\STATE Each player follows the OSO in Algorithm \ref{Online Single Oracle algorithm} with their respective effective strategy sets $\Pi_t, C_t$
\ENDFOR
\STATE {\bfseries Output}: $\vpi_T, \Pi_T,\vc_T, C_T$
\end{algorithmic}
\end{algorithm}
\vspace{-10pt}
\end{wrapfigure}
Recall that  if both players follow a no-regret algorithm, then the average strategies of both players converge to the NE in two-player zero-sum  game~\cite{cesa2006prediction,blum2007learning}. 
Since OSO algorithm has the no-regret property, it is then natural to study the self-play setting where both players apply OSO (which we call ODO) and investigate its convergence rate to the NE in large games. 

Compared to the standard DO, ODO no longer needs computing the NE in each sub-game. %, since ODO do not require the NE information in each sub-game.
%Compared to the standard DO method, ODO releases the assumption of knowing the game matrix, this is because we do not need to solve the NE of each sub-game at every iteration 
%\yaodong{You still need the matrix entries for the loss computation for the action subset $\Pi$ to compute your full information feedback. This is no different to what DO needs, as the Nash computation step of DO does not have to be LP, but could be any method that uses this full information feedback.}. \lecong{now?}
Furthermore, ODO produces rational agents as each agent can exploit their adversary to achieve the no-regret property.
%\yaodong{ODO is a self-play method, so saying this algo is rational is not right?}.\lecong{changed}
%\yaodong{echo back to the case where you introduce douvle oracle. make sure we mention this when introduce DO \lecong{do it in the morining about DO section 3.1}}. %Furthermore, the ODO algorithm is naturally motivated as each player aims to maximise its average payoff\zheng{the context of this sentence is not clear enough. Probably you want to compare it against common no regret algos?}.
For the ODO method, the convergence result to the NE is presented as follows.  
\begin{theorem}\label{ODO convergence rate}
Suppose both players apply OSO. Let $k_1$, $k_2$ denote the size of effective strategy set for each player. Then, the average strategies of both players converge to the NE with the rate: %\yaodong{consider the case for OSO-slow-add}
\[\epsilon_T= \sqrt{\frac{k_1 \log(k_1)}{2T}}+\sqrt{\frac{k_2 \log(k_2)}{2T}}.\]
In situation where both players follow OSO with Less-Frequent Best Response in Equation (\ref{eq:oasas}) and  $\alpha_{t-|\bar{T}_i|}^i=\sqrt{t-|\bar{T}_i|}$, the convergence rate to NE will be
\vspace{-5pt}
\[
\epsilon_T= \sqrt{\frac{k_1 \log(k_1)}{2T}}+\sqrt{\frac{k_2 \log(k_2)}{2T}} +\frac{\sqrt{k_1}+\sqrt{k_2}}{\sqrt{T}}.\] 
(We provide the full proof in the Appendix {\color{blue}{B.1}}.)
\end{theorem}
%The proof is built up based on the regret bound of OSO in Theorem \ref{regret of online oracle with stationary distribution}.
Theorem \ref{ODO convergence rate} suggests that
similar to OSO, the convergence rate of ODO will not depend on the   game size, but rather the size of effective strategy set of both players. As we have illustrated in Remark~\ref{small effective strategy set remark} and also in the later Figure \ref{size of k figure}, there is a linear relationship between the size of effective strategy set and the support size of the NE.
%\yaodong{linear is quite strong word}.\lecong{we make this argument through the paper} 
Furthermore, as we will show in Table 2 in Appendix {\color{blue}{C}}, in many real games, the size of the NE support is indeed much smaller than the game size. %\yaodong{but not linear?}.\lecong{the linear relationship refers to effective set and support of nash, not the size of the game}
Therefore, our ODO method can be both theoretically and empirically used as a solver in large size zero-sum games.

\textbf{Online Double Oracle \emph{vs.} Double Oracle with MWU.}
We want to highlight that ODO is markedly different from simply implementing DO by adopting  MWU  to solve sub-game NEs (i.e., run MWU till convergence in line $5$ in Algorithm \ref{Double Oracle algorithm}). 
Firstly, the MWU update of ODO happens at every iteration and ODO adds a  best response per MWU update, whereas DO adds a best response every time  a sub-game NE is solved, which often requires thousands of MWU iterations.  
Most importantly,  the best response target in ODO (i.e., the time-average loss  $\bar{\vl}$) is not necessarily a NE; this is in contrast to DO where the best response is computed with respect to an exact NE. 
In other words, even if DO implements MWU to solve the sub-game NE, it is still not a no-regret algorithm. This also explains the performance gap between ODO and DO (with MWU solving sub-game NEs) in Figure \ref{fig: game of skills}.
\subsection{Considering $\epsilon$-Best Responses}
\label{sec:abr}
So far, OSO agents require to compute the exact best response to the average loss function $\bar{\vl}$ (i.e., line $10$ in Algorithm \ref{Online Single Oracle algorithm}). Since calculating the exact best response is often computationally heavy and even infeasible in large games, an alternative way is to consider a $\epsilon$-best response (e.g., through a RL subroutine similar to PSRO \cite{lanctot2017unified}).  Here we consider the cases where the learner can only access to a $\epsilon$-best response to the average loss.  
We first present the following lemma showing that by using approximate best response, the original DO can in fact converge to an $\epsilon$-NE in Equation (\ref{eq:ene}). 
%We first provide a lemma about $\epsilon$-best response in the DO setting where the players know the game matrix $\mA$. 
%\yaodong{why we need this lemma? explain them}.
\begin{lemma}\label{lemma:dodo}
DO will converge to $\epsilon$-NE if  players can only access to an $\epsilon$-best response in each round. 

(We provide the full proof in the Appendix {\color{blue}{B.2}}.)
%Then, the algorithm converges to $\epsilon$-Nash equilibrium of the game \yaodong{the way this lemma prsents is wierd. Just say in a definitive way. XXX algorithm converges to NE if e-best response is aodpte.}
\end{lemma}
Based on Lemma \ref{lemma:dodo}, we can now derive the regret bound as well as the convergence guarantee for a OSO learner in the case of $\epsilon$-best response.
\begin{theorem}\label{theorem online Oracle with epsilon best response}
Suppose OSO agent can only access the  $\epsilon$-best response in each iteration when following Algorithm \ref{Online Single Oracle algorithm}, if the adversary follows a no-regret algorithm, then the average strategy of the agent will converge to an $\epsilon$-NE. 
%\[\max_{ \vc \in C} \bar{\vpi}^{\top} \mA \vc-\epsilon \leq \bar{\vpi}^{\top} \mA \bar{\vc} \leq \min_{\vpi \in \Pi} \vpi^{\top} \mA \bar{\vc} +\epsilon.\]
Furthermore, the algorithm is $\epsilon$-regret:
\vspace{-5pt}
\[\lim_{T \to \infty} \frac{R_T}{T}\leq \epsilon;\;\;\; R_T= \max_{\vpi \in \Delta_{\Pi}} \sum_{t=1}^T \left( \vpi_t^{\top}\mA \vc_t-\vpi^{\top} \mA \vc_t\right).\]%\yaodong{this theorem is quite important, i am surprised it is written via words, not math formulation. \lecong{will put the math in}}
(We provide the full proof in the Appendix {\color{blue}{B.3}}.)
\end{theorem}
Theorem \ref{theorem online Oracle with epsilon best response} implies that in the case of approximate best responses, OSO  learners can still  approximately converge to a NE. This results essentially authorises the application of optimisation methods to approximate the best response in the ODO process, which paves the way to use RL algorithm in solving complicated zero-sum games such as StarCraft \cite{vinyals2019grandmaster}.   
%The proof can be found in the Appendix B.3.

%\yaodong{add explanaiton of what does the theorem 8 imply}

\subsection{Considering Games with Unknown $k$}
\label{sec:large-k}
%\subsection{Prod Algorithm with Online Single Oracle algorithm} \Haitham{Try to motivate why u splitting theory results across multiple sections. These all can shrink into one section where u start it we do the prove under different considerations. Consideration I: epsilon-Best respone : Why we need it and theorem, consideration II: Something else why we need it and theorem and so on. Also all through the equations make sure u use left and right for brackets and that when you have text like support write \\text{supp} instead of just support. It will look better. One last thing I realised is that references are all over the place. U have some in the intro some in the related work and some in the other sections. Maybe u want to have them such that u put all what everybody has done and theorems and lemmas u use from them in background and use them by referring to them. }

The effectiveness of ODO is built on the key condition  that the size  of the NE support is smaller than the game size (i.e., Assumption \ref{keyassump}). 
Although many real-world zero-sum games do exhibit relatively small $k$ (see Table 2 in Appendix {\color{blue}{C}}),  in the worst-case scenarios, $k$ can be as large as the game size, in which case one can directly apply a standard no-regret method (e.g., MWU), which we denote as algorithm $B$. 
In reality, the size of $k$ is unknown  until we solve the game. 
%\yaodong{what does this mean in math}.
In the section, we offer a robust variation of ODO that account for the cases when $k$ is unknown (either large or small) and develop a no-regret algorithm that achieves the regret at least as good as ODO and algorithm $B$.  

%No-regret algorithms such as  MWU  have a fixed regret bound with no need of the  assumption on $k$. Thus, in the worst-case scenario where the assumption does not hold, we aim to achieve the regret at least as good as the normal no-regret algorithm's performance. 
To fulfil this goal, we can develop a new algorithm by combining  OSO with the algorithm $B$, hoping that the new combined  algorithm can achieve the best regret of both algorithms.  
The Prod algorithm~\cite{cesa2007improved}  provides an analytical tool to achieve such a goal. Given two algorithms for "easy" and "hard" problem instances, intuitively, the Prod offers a method to combine the two algorithms to achieve the best possible guarantees in both cases. In our setting, the "hard" problem refers to the games with large $k$ while the "easy" problem refers to games with small $k$. %Therefore,  we need a decent algorithm to solve large $k$ game to apply Prod. %\lecong{check it Yaodong}
%\yaodong{summarise your solution of Prod here, why we consider this and how it works. 2-3 sentences} 

%The OSO algorithm \ref{Online Single Oracle algorithm} will work well in many situation where the oracle algorithm requires sufficiently small number of iterations to converge.  In this section, we suggest a way to improve the regret in this scenario. The idea is to play the OSO algorithm \ref{Online Single Oracle algorithm} alongside another algorithm so that we can maintain a low regret property of the two algorithms.

%Suppose there exists an algorithm $B$ such that by following the algorithm $B$, the learner will have a regret bound $R_B(T)$ which depends on the size of strategy set (e.g., the algorithm $B$ may be hard to compute due to a large strategy set).

Suppose the no-regret algorithm $B$ has a regret-bound of $R_B(T)$. 
Then by applying the Prod algorithm, we have the OSO-Prod algorithm. 
The intuition of OSO-Prod is that, at each iteration $t$, OSO-Prod updates the weight $w_{t,B}$ for algorithm $B$ based on the relative performance with OSO in round $t$; if  positive, then in the next round $t+1$, by increasing the weight $w_{t,B}$, OSO-Prod  will follow strategy of algorithm $B$ with a higher probability. 
The  pseudocode of OSO-Prod is listed   in Appendix {\color{blue}{B.4}}. 
Here we provide the regret bound of the OSO-Prod algorithm:
\begin{theorem} \label{OSO-prod theorem}
Following the OSO-Prod algorithm with the learning rate $\eta =\frac{1}{2} \sqrt{(\log(T)/T)}$ and $w_{1,B}=1-w_{1,A}=1-\eta$ guarantees the regret bound $R_T(\text{OSO-Prod})$ %\yaodong{wierd notation?\lecong{this is from the original Prod}} 
of
\vspace{-5pt}
\begin{equation*}
\begin{aligned}
  \min \Big(R_B(T)+ 2\sqrt{T \log(T)},{ \frac{T \sqrt{k \log(k)}}{2}} +2\log(2)\Big).
\end{aligned}
%\vspace{-5pt}
\end{equation*}
(We provide the full proof in the Appendix {\color{blue}{B.4}}.)
\end{theorem}
\vspace{-5pt}
%The proof of Theorem 
%%\ref{OSO-prod theorem} \yaodong{this only deserve a lemma?} 
%comes from the regret bound theorem of the Prod algorithm \cite{sani2014exploiting}. 
By following OSO-Prod, the learner can still achieve a low regret bound when $k$ is small while maintain a  regret that is at least as good as the algorithm $B$ in the cases of large $k$. 
%\zheng{if not implemented, not sure we should have it in the main body especially when we have very limited space. leave this for discussion with yaodong for the next meeting. \lecong{let s discuss it tomorrow if we have time}}
%In the next section, we will study online markov decision processes and the application of OSO algorithm in this setting. 
\vspace{-0pt}
\section{Experiments \& Results}\label{section experiment}
\vspace{-5pt}

%$\mathcal{O}\big(n^{2.38} \log({n}/{\delta})\big)$ 

%\lecong{List of experiments we need to do}:
%\begin{enumerate}
%    \item Convergence of OSO with different matrix sizes (done)
%    \item Online Single Oracle vs MW in the case the column player follows MW with asymmetric matrix(10000x 3) (done)
%    \item Size of an effective set with different sizes of row player's strategy(done)
%    \item Performance of Online Double Oracle compared to MW and FP baseline in Alphastar(done)
%    \item Performance of Online Double Oracle compared to MW and FP baseline in different symmetric games(done: 18 games).
%    \item Performance of ODO vs PSRO in Poker games in term of exploitability and facing "easy" adversary (half did: waiting for result about easy adversary)\lecong{still watiing for Nicholas to run this experiments}
%\end{enumerate}
In this section, we aim to demonstrate the effectiveness of our OSO and  ODO methods. Firstly, we verify the small support size of NEs in Assumption \ref{small support size assumption} through random matrix games.
Later, on tens of complicated real-world games, we evaluate our OSO methods  in both self-play settings and playing against  strategic adversaries. 
Finally, we show the performance of ODO on large-scale Poker games. As we benchmark on the number of iterations against other baselines, for fair comparisons, we implement the plain OSO in Algorithm \ref{Online Single Oracle algorithm} without the $\alpha$ threshold trick mentioned in Equation \ref{eq:oasas} for all our experiments. All hyperparameter settings can be found in Appendix \textcolor{blue}{C}. 
%\yaodong{mentione where did we use the slow best responsing variant of ODO}
%\yaodong{mention where can we find all the hyper paramete setting}
%and that theoretical improvement in regret bound and convergence rate of OSO/ODO will lead to higher average payoff and faster convergence to NEs in practice. 
%\yaodong{explain more on what exp settings we have used, and what is main assumption we want to verify}\lecong{now?}

\begin{wrapfigure}{R}{0.5\textwidth}
\vspace{-20pt}
  \begin{center}
\centerline{\includegraphics[width=0.48\textwidth]{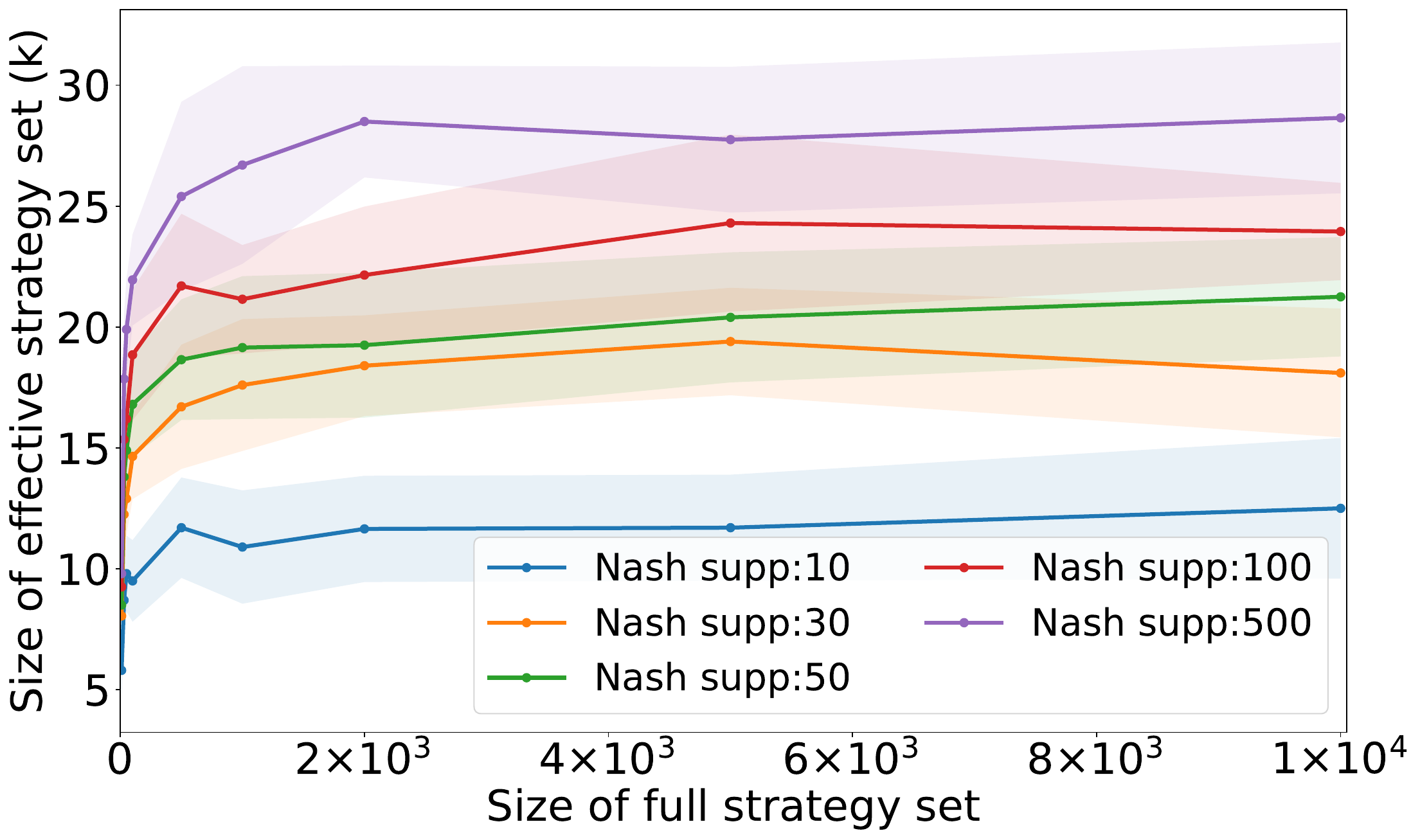}}
\vspace{-5pt}
\caption{Sizes of effective strategy set (i.e., $k$) in cases of an OSO agent playing against an MWU opponent with different  sizes of  full strategy set and NE support.} 
%\yaodong{change n-col to Nash support size}}

\label{size of k figure}
\end{center}
\vspace{-20pt}
\end{wrapfigure}
%\begin{enumerate}
%    \item What is the relationship between the size of effective strategy set $k$ and the corresponding full strategy set size?
%    \item How do our algorithms perform in real-world zero-sum games in comparison with other strong baselines?
%    \item Is the assumption that the support size of a NE in a zero-sum is often smaller than the game size general in real-world problems?
%    \item Facing large-scale problems where normal no-regret algorithms are infeasible, how do our algorithm perform compared to other feasible solutions?
%    \item Can OSO with no-regret property exploit a restricted opponent?
%\end{enumerate}

\textbf{Empirical study on the size of $k$ vs.  the full strategy set.}
We consider a set of zero-sum normal-form games of different sizes,   the entries of which are sampled from a uniform distribution $\textbf{U}(0,1)$. %, allowing us to vary the full game size. 
We run OSO as the row player against a MWU column player  until convergence, and plot the size of the OSO player's effective strategy set against its full strategy size.  
We run $20$ seeds for each  setting. 
As we can see from Figure \ref{size of k figure}, given a fixed support size of NE\footnote{We achieve this by fixing the number of columns while increasing the number of rows in the game matrix.},  the size of the effective strategy set $k$ grows as the number of full strategy set increases, but plateaus quickly. The larger the size of NE support (not the full strategy set!), the higher this plateau will reach. Clearly, we can tell that the size of OSO's  effective strategy set does not  increase with the full strategy size, but rather depends on the support size of a NE. %\footnote{By Lemma \ref{lemma bound on k}, the support size of the NE is as large as the number of columns. Thus, when we fixed the number of column, the support size of the NE is stable.}. 
This result confirms Theorem \ref{regret of online oracle with stationary distribution} in which we prove that OSO's regret bound depends on $k$ that is related to the size of NE support but  not the game size.  Economically, this is a desired property as OSO can potentially avoid unnecessary computations in contrast to other no-regret methods that require looping  over the full strategy set at each iteration.  
%\begin{figure}[t!]
%\begin{center}
%\centerline{\includegraphics[width=0.5\textwidth]{figures/proportion.pdf}}
%\vspace{-10pt}
%\caption{Sizes of effective strategy set (i.e. the $k$) learned by OSO against MWU with different sizes of row player's full strategy set. \yaodong{try to sue the wrap-up environment to save space}}
%\label{size of k figure}
%\end{center}
%\vskip -0.1in
%\end{figure}

\begin{figure*}[t!]
     \centering
\vspace{-25pt}
\begin{subfigure}[l]{.49\textwidth}
         \centering
         \includegraphics[width=1.\textwidth]{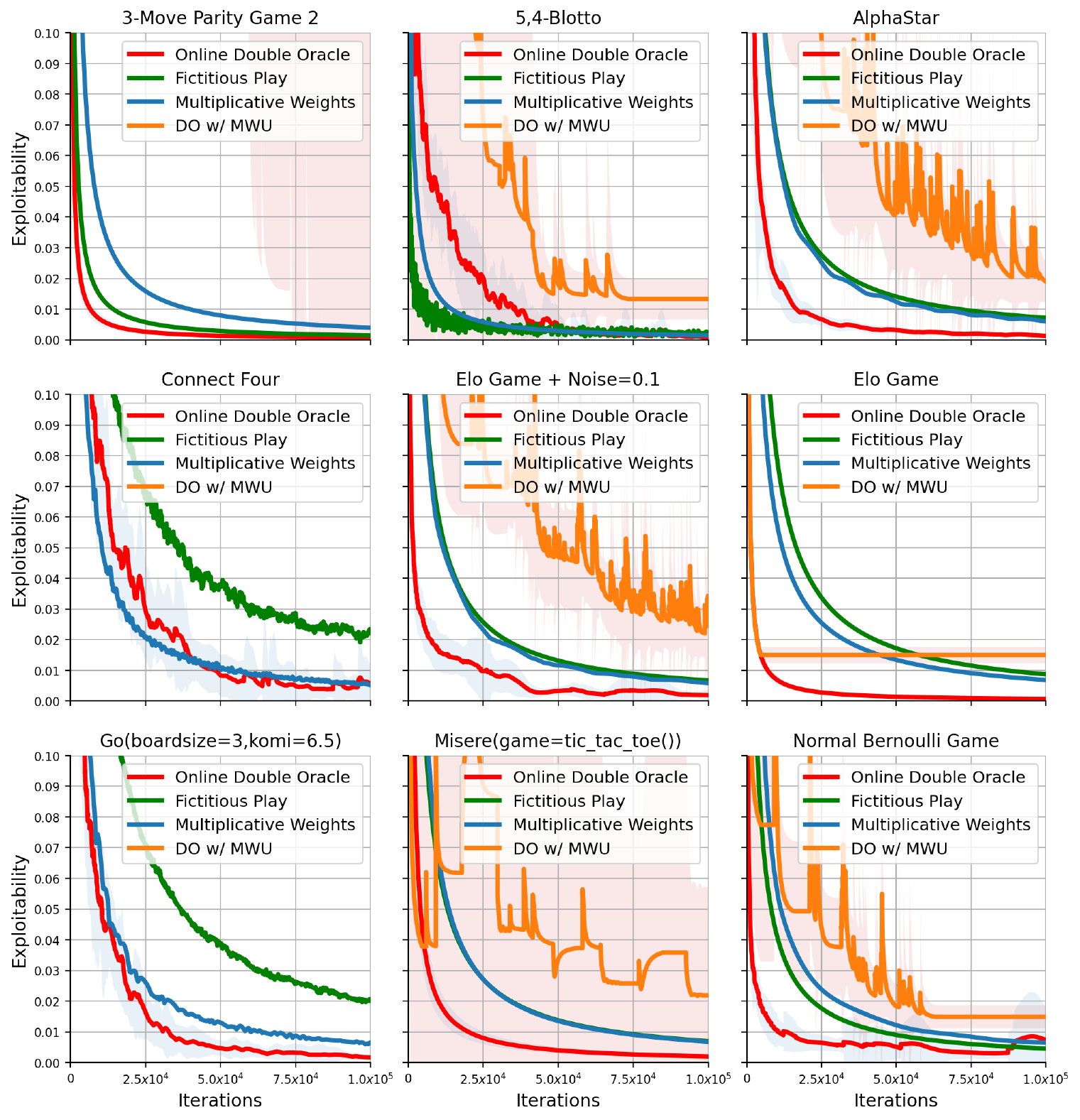}
         \vspace{-12pt}
\caption{Performance comparisons under self-plays} %\stephen{how can we compare CFR on the same graph with the x asis number of BRs computed?}}
         \label{fig:Performance of ODO against different Games}
\end{subfigure}
\begin{subfigure}[l]{.49\textwidth}
         \centering
         \includegraphics[width=1.\textwidth]{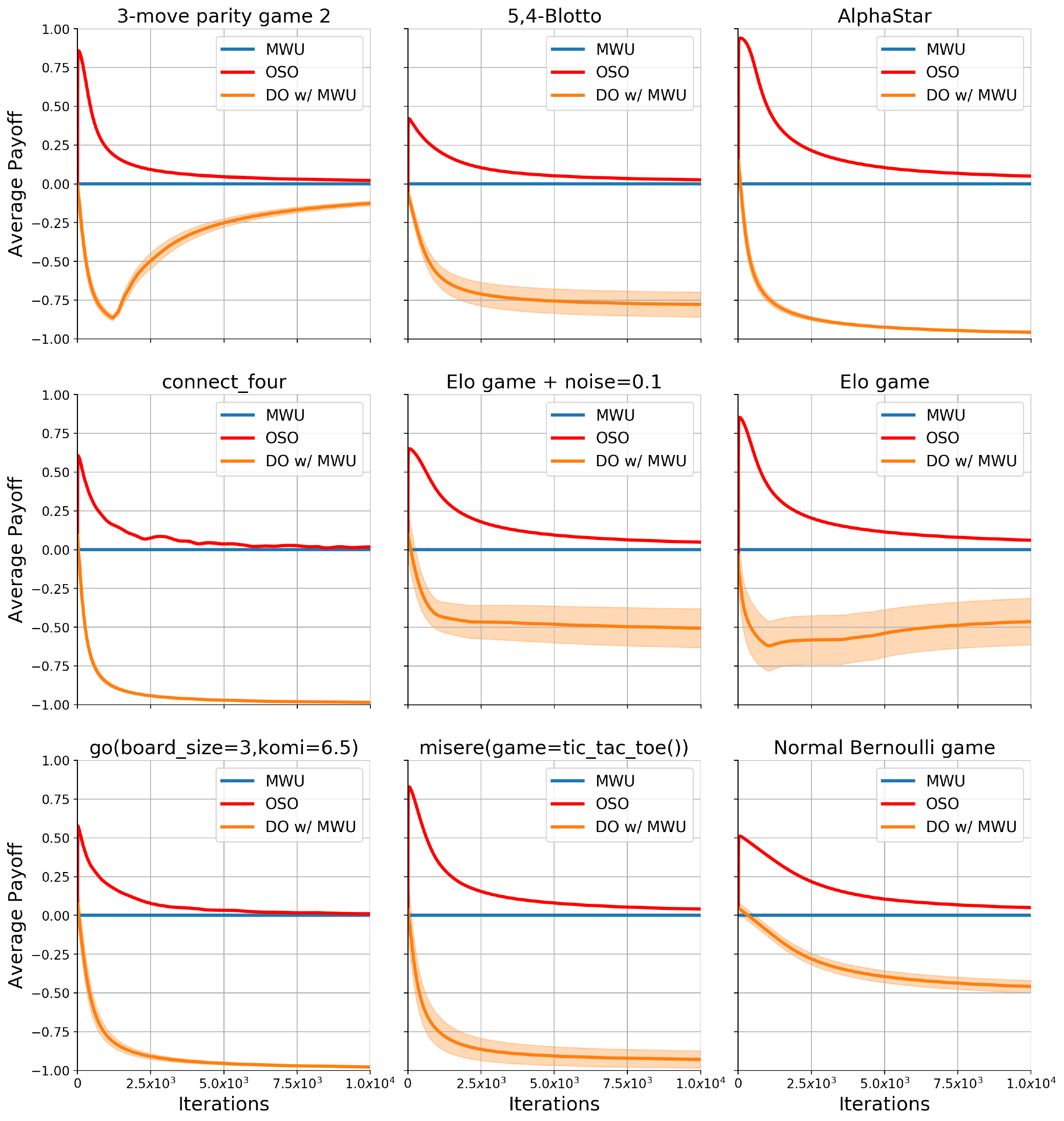}
        \vspace{-15pt}
\caption{Payoffs of playing against an MWU adversary}
       \label{fig:average_payoff_against_MWU}
\end{subfigure}
 \vspace{-5pt}
\caption{Performance comparisons on real-world games.  Full $30$ results are reported in Appendix {\color{blue}{C}}}
%($3^{936}$ pure strategies).}
  \label{fig: game of skills}
 \vspace{-10pt}
\end{figure*}

%\begin{figure*}[t!]
%\vspace{5pt}
%\begin{center}
%\centerline{\includegraphics[width=0.7\textwidth]{figures/2x5_final.pdf}}
%\vspace{-10pt}
%\caption{Performance comparison of exploitability on Real-world Zero-Sum Games. Full results are reported in Appendix {\color{blue}{C}}.  
%\yaodong{mention why DO is not mentioned, add one ablation study} \lecong{is it to big now?}
%}
%\label{fig:Performance of ODO against different Games}
%\end{center}
%\vspace{-25pt}
%\end{figure*}

%\begin{figure*}[t!]
%\vspace{-15pt}
%\begin{center}
%\centerline{\includegraphics[width=0.7\textwidth]{figures/average_payoff_against_MWU_2x5.pdf}}
%\vspace{-10pt}
%\caption{Performance comparison of average payoff against an MWU adversary on Real-world Zero-Sum Games. Full results are reported in Appendix {\color{blue}{C}}.  
%\yaodong{mention why DO is not mentioned, add one ablation study} \lecong{is it to big now?}
%}
%\label{fig:average_payoff_against_MWU}
%\end{center}
%\vspace{-25pt}
%\end{figure*}
% We consider two different sets of games: asymmetric and symmetric games. In asymmetric games with randomly generated matrix $\mA$, we report the relationship between the effective strategy set and the Nash equilibrium support size. This will validate our small support size NE assumption in application such as OMDPs. In symmetric games,

\textbf{Performance on Real-World Zero-Sum Games.} We investigate ODO in terms of convergence rate to a NE. To demonstrate applicability to real-world problems, we replace the randomly generated normal-form games with $15$ popular real-world zero-sum games from Czarnecki et. al. \cite{czarnecki2020real}. We compare the \emph{exploitability} \cite{davis2014using} (i.e., the distance to a true NE) %\ollie{the metric for these games is not exploitability, it is $\epsilon$-Nash right?\zheng{yes, you are right, Lecong should we just say convergence rates or something else?}} \lecong{We can delete the exploiability, I have mentioned the relationship between Ne and exploiability in footnote 5} 
of ODO with other  baseline methods (MWU, FP and DO\footnote{Following the discussion in Section \ref{Online Double Oracle} and for fair comparisons, we implement DO by adopting MWU as  the sub-game NE solver and  report the total number of MWU iterations  DO needs to achieve  a low exploitability.}).  
%Since we do not assume any priority knowledge about the game matrix $A$, DO with LP~\cite{mcmahan2003planning} can not be applied.}
We  run each game with $20$ seeds. As it shows in Figure \ref{fig:Performance of ODO against different Games},
%\zheng{it's bit strange that we only introduce fig 2 and 3 after fig4}\lecong{changed},
ODO outperforms the baselines in almost all $15$ games. 
The advantages of ODO in terms of convergence rate over MWU and FP match our expectation as the support sizes of the NEs in these games are much smaller than the game sizes (reported in Table 2 in Appendix {\color{blue}{C}}). For DO with MWU, since it takes many iterations in each sub-game to converge to the NE, it performs much worse than ODO.

\begin{figure*}[t!]
     \centering
\vspace{-25pt}
\begin{subfigure}[l]{.49\textwidth}
         \centering
         \includegraphics[width=1.\textwidth]{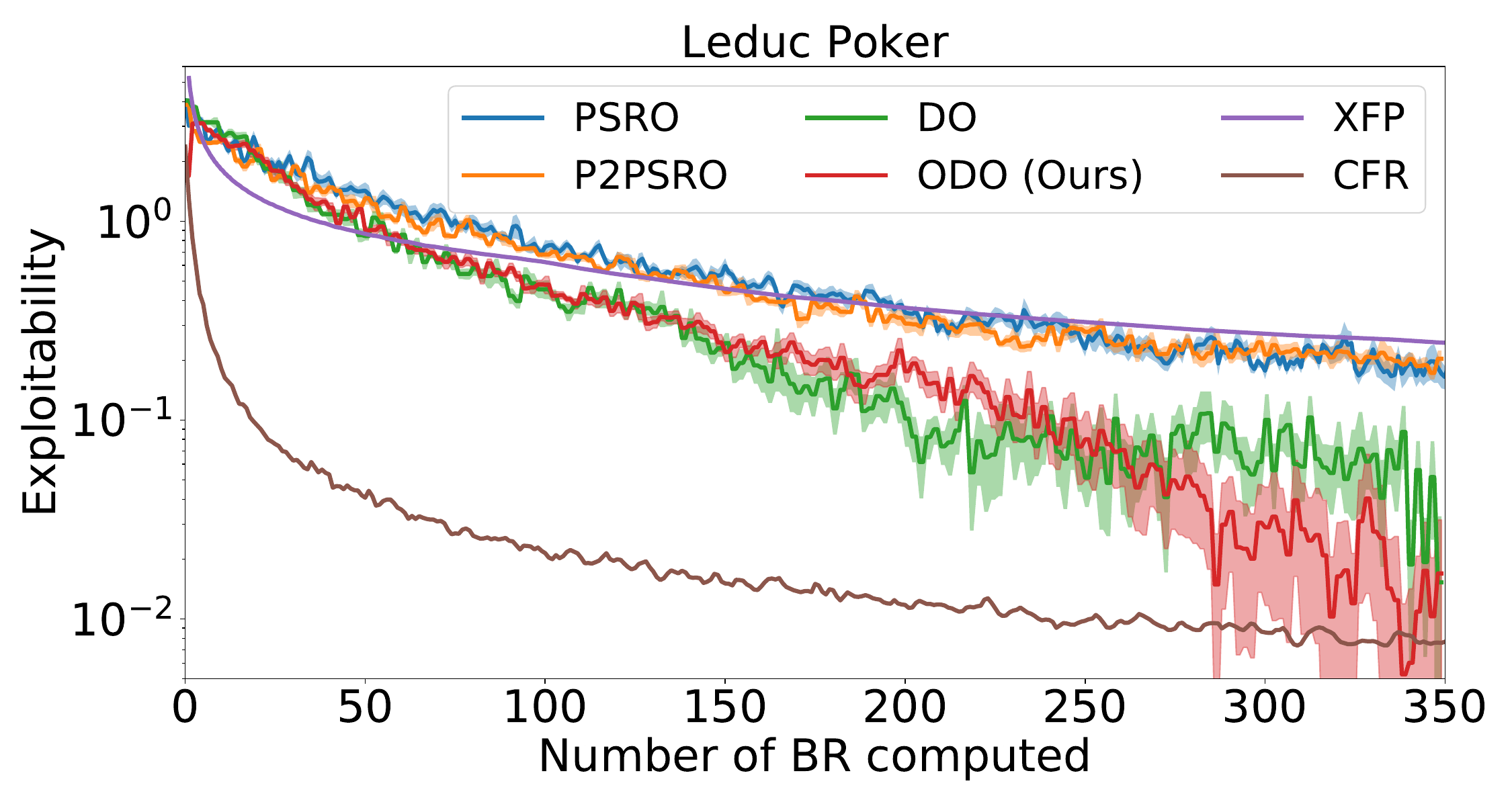}
         \vspace{-15pt}
\caption{Exploitability on Leduc Poker} %\stephen{how can we compare CFR on the same graph with the x asis number of BRs computed?}}
         \label{fig:Leduc Poker experiment in exploiability}
\end{subfigure}
\begin{subfigure}[l]{.49\textwidth}
         \centering
         \includegraphics[width=1.\textwidth]{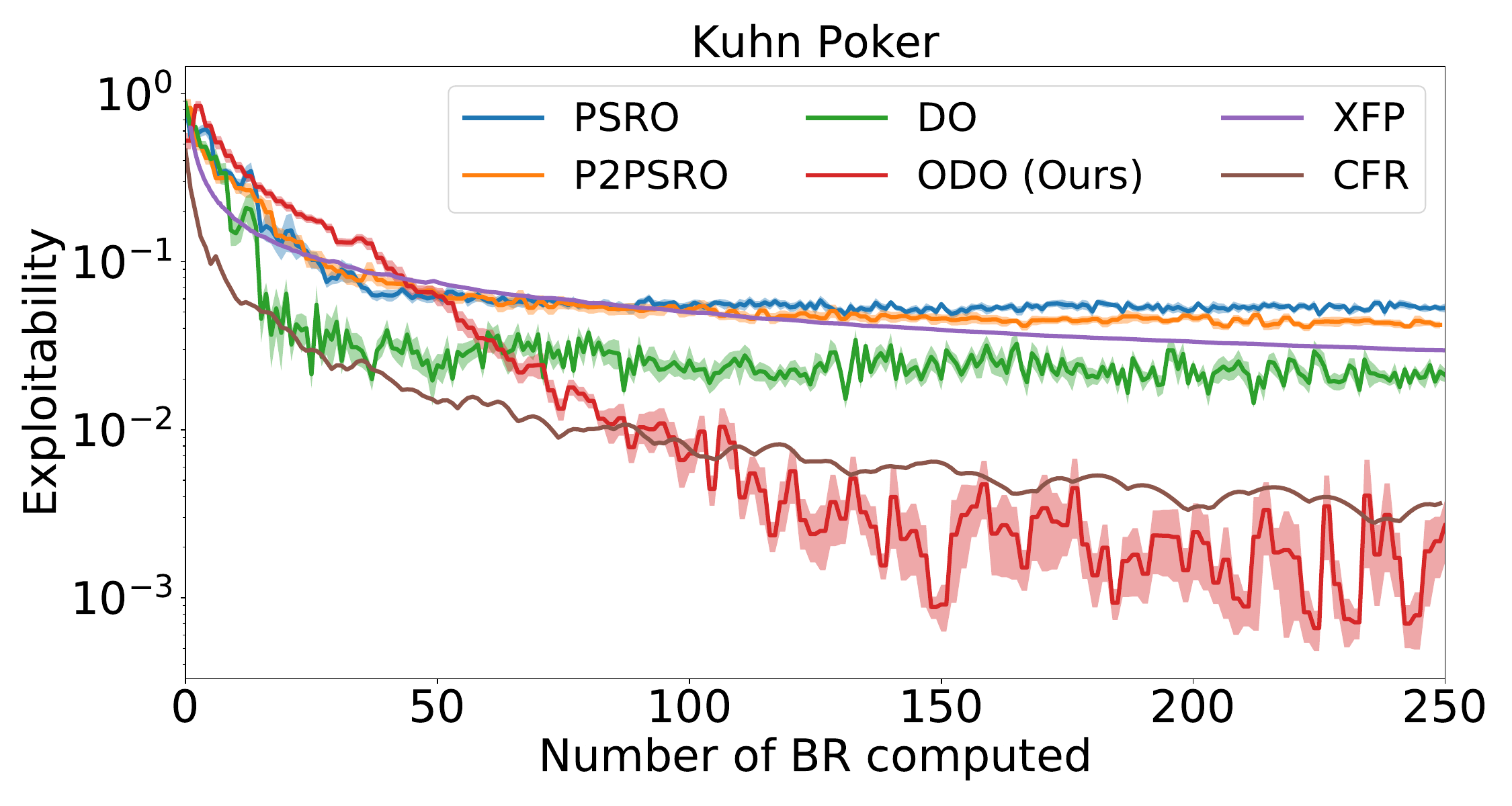}
                  \vspace{-15pt}
\caption{Exploitability on Kuhn Poker}
       \label{fig:kuhn Poker experiment in average payoff}
\end{subfigure}
 \vspace{-3pt}
\caption{Performance comparisons in exploitability on Poker games.}
%($3^{936}$ pure strategies).}
  \label{fig:exploitability}
 \vspace{-15pt}
\end{figure*}

Apart from the self-play setting, we also conduct the setting of playing against an MWU adversary in Figure \ref{fig:average_payoff_against_MWU}. 
%On average performance, we test the performance of OSO against MWU and DO when facing the same opponent (i.e., the opponent follows the MWU). As we can see in Figure...\lecong{add the figure here soon}, 
We can tell that OSO outperforms MWU and DO baselines in average performance in almost all  $15$ games, which confirms the effectiveness of our design.
Notably,  
MWU achieves a constant payoff; we believe it is because these  games are symmetric and since both players follow MWU with the same learning rate,  the payoff will always be the value of the game (thus the ground truth), which  OSO will eventually converge to as well.

\textbf{Performance on Poker Games.}
%\yaodong{rewrite this section by not saying we compare the sota on leduc, but show the effectiness of our algo on leduc, and compare against tabular version. check the rebuttal}\lecong{Now?}
%Though adapted from real games, the above $15$  games are not large enough such that normal no-regret algorithms are inapplicable. 
To further investigate ODO's effectiveness, we test  ODO on Kuhn Poker and Leduc Poker. %, which contain $2^{12}$ and $3^{936}$ pure strategies respectively.
Since ODO is designed only for normal-form games, 
 we adopt the  tabular settings \cite{mcaleer2020pipeline, lanctot2017unified} in which an exact best response is computed by a tree-traverse oracle (also see OpenSpiel \cite{lanctot2019openspiel}), and for PSRO methods, we perturb the exact best response with random noises. 
 We benchmark how many times such a  best-response oracle is called by different methods.   
 We compare against the state-of-the-art PSRO method: P2PSRO\footnote{We have discounted the fact that P2PSRO uses multiple workers (we use two) to compute best responses.  }~\cite{mcaleer2020pipeline}, and   two extensive-form game solvers CFR \cite{zinkevich2007regret}  and XFP \cite{lanctot2019openspiel}.   
 %To make fair comparisons,The MWU and FP baselines we use in previous games will not work well in these large games due to the computational hardness . Instead, we use the best response oracle algorithms in solving Leduc Poker as the baselines (i.e., Double Oracle~\cite{even2009online} \yaodong{explain},PSRO~\cite{lanctot2017unified}, Pipeline PSRO~\cite{mcaleer2020pipeline}). We note that since Kuhn Poker and Leduc Hold~'em can be efficiently represented in extensive-form, the state-of-art algorithms to solve them are counterfactual regret (CFR) base methods. We use the normal-form realization of these games only to demonstrate the scalability of ODO in large size games. We include the experiment result of CFR and Extensive Fictitious Play (XFP)~\cite{lanctot2019openspiel} for completeness of the paper.  %We also include the CFR/XFP baselines in tabular setting as mentioned in ~\cite{lanctot2019openspiel}. 
As it shows in Figure \ref{fig:exploitability}, ODO  
%~\footnote{We run our setting over 10 random seeds. We use a new initial strategy $\vpi_t= \frac{i-1}{i}{\bar{\vpi}}_{T_i}+\frac{1}{i}\va$ in step 8 of Algorithm \ref{Online Single Oracle algorithm}, where ${\bar{\vpi}}_{T_i}$ is the average strategy of in the last time window $T_i$ and $\va$ is the new best response.} 
shows a significant improvement in exploitability compared to all existing DO and PSRO baselines, and it almost catches up with the state-of-the-art solver CFR in Leduc Poker, and it outperforms CFR on Kuhn Poker.   
Importantly, ODO uses the less number of best-response calls to achieve the  lowest exploitability. 
We believe they are strongly promising results, since taking the most effective use of best-response functions has a critical   impact on solving complex games (e.g.,  StarCraft \cite{vinyals2019grandmaster}).   
%, making it less computationally expensive.
% \stephen{I don't like this paragraph. Don't we compare to CFR in figure 3 a? Also, it's surprising to me that we do better than CFR on Leduc, CFR is really hard to beat. But I do think we should include experiments that show how this approach can actually be better than CFR. One experiment that I think could work is a game I used in the XDO paper where I cloned each Leduc action say 50 times. Then CFR is really slow, but we shouldn't be because we don't need to do regret matching on the whole game, just a fraction of it.}
% \lecong{now?}

\begin{wrapfigure}{R}{0.4\textwidth}
%\begin{figure*}[H]
%\begin{figure}
\vspace{-25pt}
\begin{subfigure}[l]{0.4\textwidth}
  \begin{center}
\centerline{\includegraphics[width=1.\textwidth]{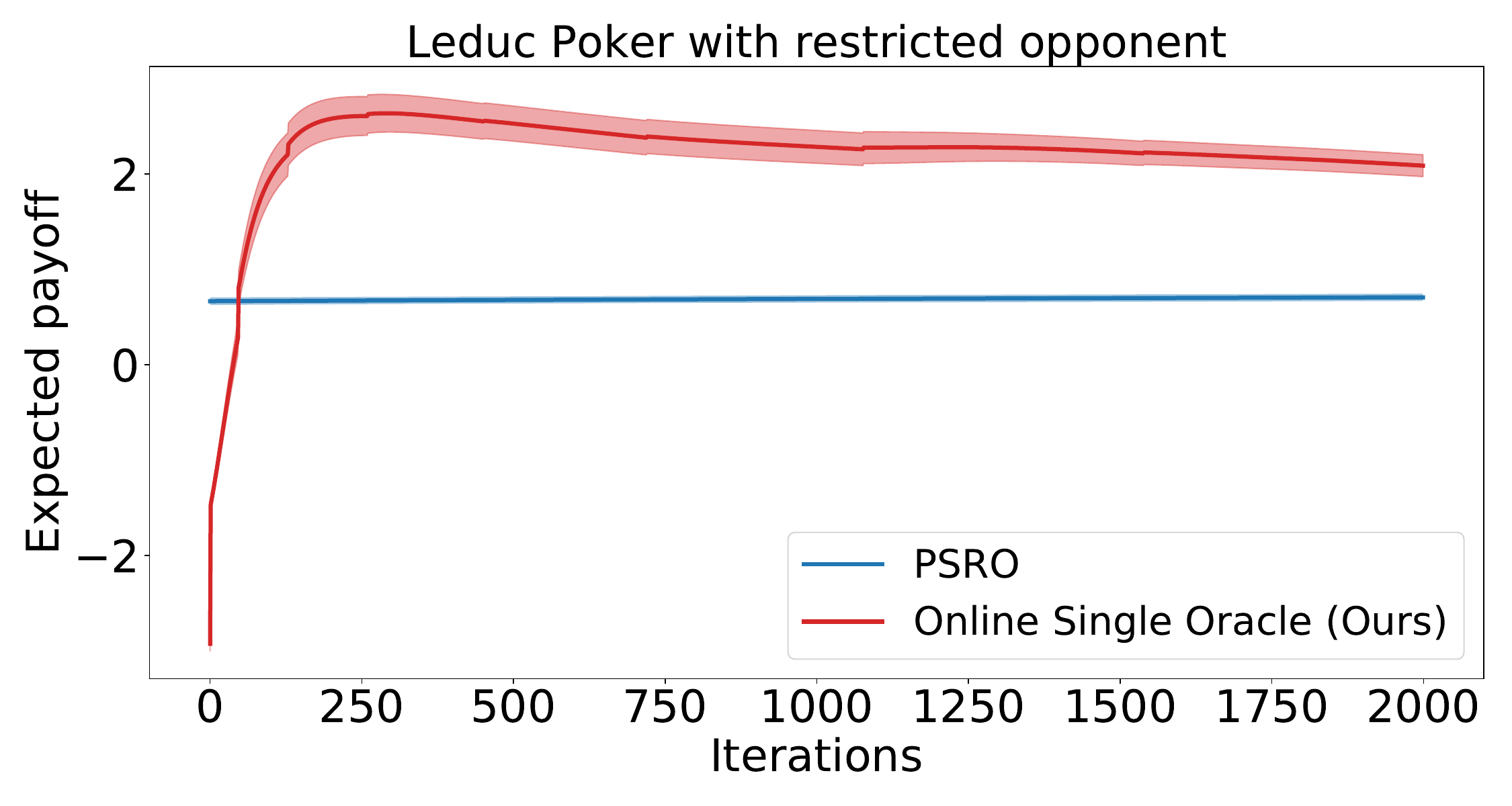}}
\vspace{-8pt}
\caption{Leduc Poker}
       \label{fig:Leduc Poker experiment in average payoff}
\end{center}
\vspace{-15pt}
\end{subfigure}
\vspace{-10pt}
\begin{subfigure}[l]{0.4\textwidth}
  \begin{center}
\centerline{\includegraphics[width=1.\textwidth]{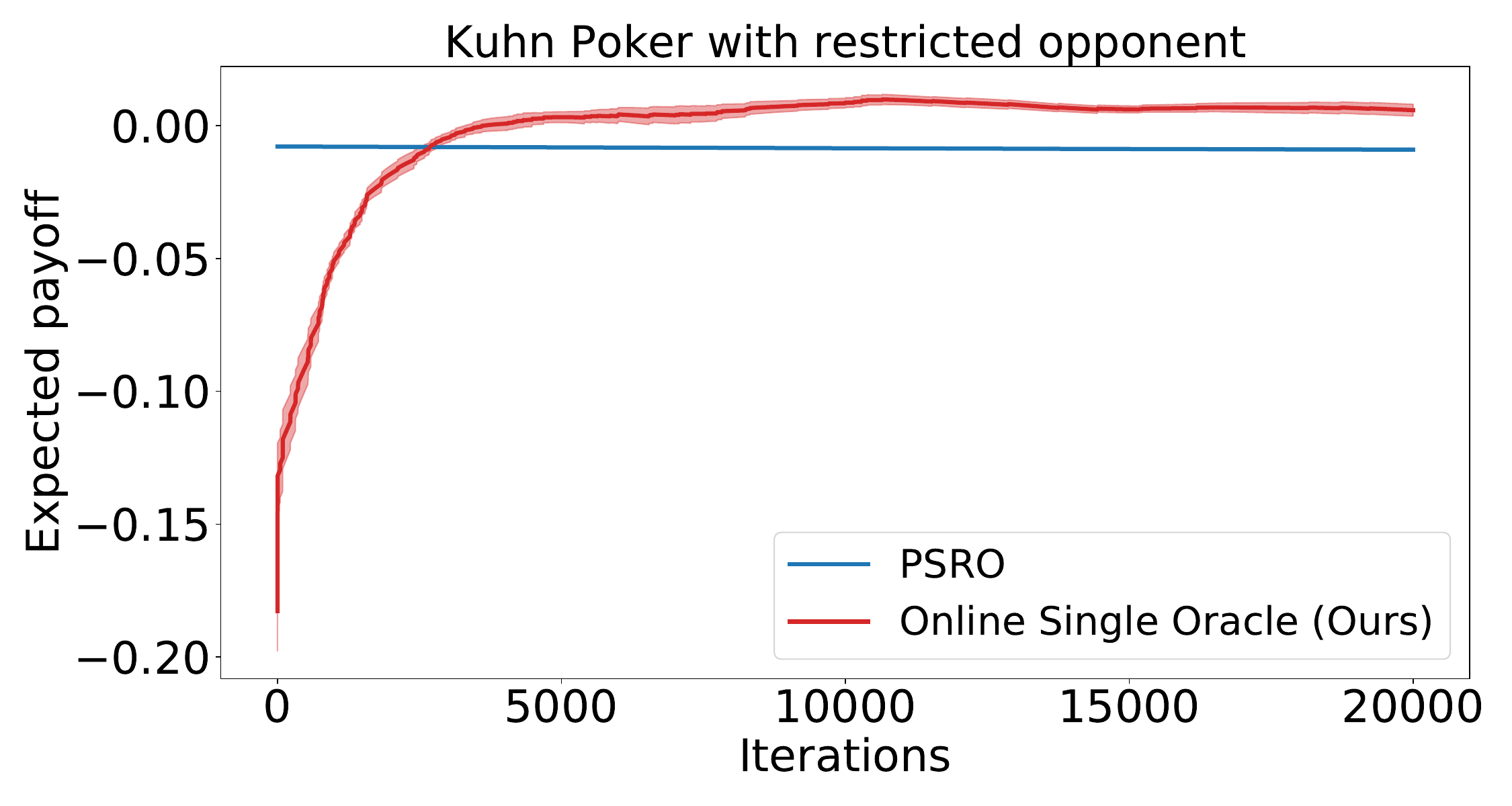}}
\vspace{-8pt}
\caption{Kuhn Poker}
       \label{fig:Kuhn Poker experiment in average payoff}
\end{center}
\end{subfigure}

\caption{Performance of playing against an imperfect opponent. } \label{fig:pokerres}
\vspace{-25pt}
%\end{figure}
\end{wrapfigure}

\textbf{Can OSO Exploit an \emph{Imperfect}  Opponent?} Finally, we want to test the no-regret property of OSO when the opponent is \emph{imperfect} and it plays a restricted set.  
%By rational, we mean a player can exploit its opponent when it's possible, e.g., when an opponent follows an algorithm that will never lead to a NE. 
We created such an opponent in both Pokers   by restricting its strategy set to only  $20$ pure strategies, and then let it play MWU. 
%For the baseline, the row player follows the $\epsilon$-NE  calculated by PSRO. 
We apply our OSO with $\epsilon$-best response as the row player, and compare its perform against PSRO. 
We run each setting for $10$ random seeds. As we can see from Figure~\ref{fig:pokerres}, 
OSO quickly achieves positive expected payoff and outperforms  PSRO, which is expected because OSO can actively exploit its opponent while PSRO behaves  conservatively by playing NE and not exploiting (thus achieving almost constant payoff).  %.equilibrium in average expected payoff. 
%; this can be explained as follow: The $\epsilon$-NE learned by PSRO will perform positively regardless of the type of the opponents, however in the case of the restricted player which cannot thoroughly learn to exploit the opponent, it will be outperformed by OSO which can exploit it accordingly. 
Notably, the average expected payoff of OSO  decreases slightly in later iterations, we believe it is because the opponent is following MWU, which is also a no-regret method, and both players will converge towards an NE of such a restricted game.  
%\yaodong{where???} 
%Since the column player follows MWU, it will learn the optimal policy with respect to the restricted pure strategy set in later iterations, thus the expected payoff of OSO decreases . 
\vspace{-10pt}
\section{Conclusion}
\vspace{-10pt}
We propose a new solver for two-player zero-sum games where the number of pure strategies $n$ is huge. % or even infinite. 
Our method, \emph{Online Double Oracle}, absorbs the benefits from both online learning and Double Oracle methods; it achieves the  regret bound of $\mathcal{O}(\sqrt{T k \log(k)})$ where $k$ only depends on the support size of NE rather than $n$. 
Unlike DO, ODO can exploit opponents during the game play. 
In tens of  real-world games, we show that ODO outperforms a series of algorithms including MWU, DO and PSRO methods  on both convergence rate to NE and average payoff against strategic adversaries.    %and Pipeline PSRO by a large margin on lowering exploitability and average payoff. 

%\yaodong{check all bib files, and don;t cite arxiv versions!}\lecong{I have changed it}
\clearpage 
\bibliographystyle{plain}
\bibliography{main}

\clearpage


\section*{Supplementary material}

# Supplementary Material for Online Double Oracle

## Contents

# A Online Single Oracle

## A.1 Proof of Lemma 1

**Lemma.** *In asymmetric games $A_{n\times m}, n \gg m$, if the NE $(\pi^*, c^*)$ is unique, then the support size of the NE will follow $|\text{supp}(\pi^*)| = |\text{supp}(c^*)| \leq m$*

*Proof.* Since the size of $\pi^*$ and $c^*$ are $n$ and $m$ respectively, the size of the support of NE can not exceed the size of the game:

$$|\text{support}(\pi^*)| \leq n; \ |\text{support}(c^*)| \leq m.$$

In the case the game $A$ has a unique Nash equilibrium, following Theorem 1 in [1], we have:

$$|\text{support}(\pi^*)| = |\text{support}(c^*)| \leq \min(n, m) = m.$$

Thus, we have proved the lemma. □

## A.2 Proof of Theorem 4

**Theorem** (Regret Bound of OSO). *Let $l_1, l_2, \ldots, l_T$ be a sequence of loss vectors played by an adversary, and $\langle \cdot, \cdot \rangle$ be the dot product, OSO in Algorithm 2 is a no-regret algorithm with*

$$\frac{1}{T}\Big(\sum_{t=1}^{T}\langle \pi_t, l_t\rangle - \min_{\pi \in \Pi}\sum_{t=1}^{T}\langle \pi, l_t\rangle\Big) \leq \frac{\sqrt{k\log(k)}}{\sqrt{2T}},$$

*where $k = |\Pi_T|$ is the size of effective strategy set in the final time window.*

*Proof.* W.l.o.g, we assume the player uses the MWU as the no-regret algorithm and starts with only one pure strategy in $\Pi_0$ in Algorithm 2. Since in the final time window, the effective strategy set has k elements, there are exactly $k$ time windows. Denote $|T_1|, |T_2|, \ldots, |T_k|$ be the lengths of time windows during each of which the subset of strategies the no-regret algorithm considers does not change. In the case of finite set of strategies, $k$ will be finite and we have

$$\sum_{i=1}^{k}|T_k| = T. \tag{1}$$

In the time window with length $|T_i|$, following the regret bound of MWU in definition 3 we have:

$$\sum_{t=|\bar{T}_i|+1}^{|\bar{T}_{i+1}|}\langle \pi_t, l_t\rangle - \min_{\pi \in \Pi_{|\bar{T}_i|+1}}\sum_{t=|\bar{T}_i|+1}^{|\bar{T}_{i+1}|}\langle \pi, l_t\rangle \leq \sqrt{\frac{|T_i|}{2}\log(i)}, \quad \text{where } |\bar{T}_i| = \sum_{j=1}^{i-1}|T_j|. \tag{2}$$

In the time window $T_i$, we consider the full strategy set when we calculate the best response strategy in step 11 of Algorithm 2 and it stays in $\Pi_{|\bar{T}_i|+1}$. Therefore, the inequality (2) can be expressed as:

$$\sum_{t=|\bar{T}_i|+1}^{|\bar{T}_{i+1}|}\langle \pi_t, l_t\rangle - \min_{\pi \in \Pi}\sum_{t=|\bar{T}_i|+1}^{|\bar{T}_{i+1}|}\langle \pi, l_t\rangle \leq \sqrt{\frac{|T_i|}{2}\log(i)}. \tag{3}$$

Sum up the inequality (3) for $i = 1, \ldots k$ we have:

$$\sum_{i=1}^{k}\sqrt{\frac{|T_i|}{2}\log(i)} \geq \sum_{t=1}^{T}\langle \pi_t, l_t\rangle - \sum_{i=1}^{k}\min_{\pi \in \Pi}\sum_{t=|\bar{T}_i|+1}^{|\bar{T}_{i+1}|}\langle \pi, l_t\rangle$$

$$\geq \sum_{t=1}^{T}\langle \pi_t, l_t\rangle - \min_{\pi \in \Pi}\sum_{i=1}^{k}\sum_{t=|\bar{T}_i|+1}^{|\bar{T}_{i+1}|}\langle \pi, l_t\rangle \tag{4a}$$

$$= \sum_{t=1}^{T}\langle \pi_t, l_t\rangle - \min_{\pi \in \Pi}\sum_{t=1}^{T}\langle \pi, l_t\rangle$$

$$\Rightarrow \sqrt{\frac{Tk\log(k)}{2}} \geq \sum_{t=1}^{T}\langle \pi_t, l_t\rangle - \min_{\pi \in \Pi}\sum_{t=1}^{T}\langle \pi, l_t\rangle. \tag{4b}$$

13

Inequality (4a) is due to $\sum \min \leq \min \sum$. Inequality (4b) comes from Cauchy-Schwarz inequality and Stirling' approximation. Thus, we have the derived regret. $\square$

### A.3 OSO with Less-Frequent Best Response

---
**Algorithm 1:** OSO with Less-Frequent Best Response
---
1: **Input:** A set $\Pi$ pure strategy set of player
2: $\Pi_0 := \Pi_1$: initial set of effective strategies;
3: **for** $t = 1$ to $\infty$ **do**
4:    **if** $\Pi_t = \Pi_{t-1}$ **then**
5:       Following the MWU update in Equation (5)
6:    **else if** $\Pi_t \neq \Pi_{t-1}$ **then**
7:       Start a new time window $T_{i+1}$
8:       Reset the MWU update in Equation (5) with a new initial strategy $\pi_t$
9:    **end if**
10:   Observe $l_t$ and update the average loss in the current time window $T_i$
      $\bar{l} = \frac{1}{|T_i|} \sum_{\pi_t \in T_i} l_t$ ;
11:   Calculate the best response:
      $a = \arg\min_{\pi \in \Pi} \langle \pi, \bar{l} \rangle$,
12:   **if** $\min_{\pi \in \Pi_{|\bar{T}_i|+1}} \langle \pi, \sum_{j=|\bar{T}_i|}^{t} l_j \rangle - \min_{\pi \in \Pi} \langle \pi, \sum_{j=|\bar{T}_i|}^{t} l_j \rangle \geq \alpha^i_{t-|\bar{T}_i|}$ **then**
13:      Update the strategy set: $\Pi_{t+1} = \Pi_t \cup a$
14:   **else**
15:      $\Pi_{t+1} = \Pi_t$
16:   **end if**
17:   Output the strategy $\pi_t$ at round $t$ for the player
18: **end for**
---

**Theorem** (Regret Bound of OSO with Less-Frequent Best Response). *Let $l_1, l_2, \ldots, l_T$ be a sequence of loss vectors played by an adversary. Then, OSO in Algorithm 1 is a no-regret algorithm with:*

$$\frac{1}{T}\Big(\sum_{t=1}^T \langle \pi_t, l_t \rangle - \min_{\pi \in \Pi} \sum_{t=1}^T \langle \pi, l_t \rangle\Big) \leq \frac{\sqrt{k \log(k)}}{\sqrt{2T}} + \frac{\sum_{i=1}^k \alpha^i_{|T_i|}}{T},$$

*where $k = |\Pi_T|$ is the size of effective strategy set in the final time window.*

*Proof.* W.l.o.g, we assume the player uses the MWU as the no-regret algorithm and starts with only one pure strategy in $\Pi_0$ in Algorithm 2. Since in the final time window, the effective strategy set has k elements, there are exactly $k$ time windows. Denote $|T_1|, |T_2|, \ldots, |T_k|$ be the lengths of time windows during each of which the subset of strategies the no-regret algorithm considers does not change. In the case of finite set of strategies, $k$ will be finite and we have

$$\sum_{i=1}^k |T_k| = T.$$

In the time window with length $|T_i|$, following the regret bound of MWU we have:

$$\sum_{t=|\bar{T}_i|+1}^{|\bar{T}_{i+1}|} \langle \pi_t, l_t \rangle - \min_{\pi \in \Pi_{|\bar{T}_i|+1}} \sum_{t=|\bar{T}_i|+1}^{|\bar{T}_{i+1}|} \langle \pi, l_t \rangle \leq \sqrt{\frac{|T_i|}{2} \log(i)}, \quad \text{where } |\bar{T}_i| = \sum_{j=1}^{i-1} |T_j|. \quad (5)$$

Since in the time window $T_i$, the size of the effective strategy set does not change, thus we have:

$$\min_{\pi \in \Pi_{|\bar{T}_i|+1}} \sum_{t=|\bar{T}_i|+1}^{|\bar{T}_{i+1}|} \langle \pi, l_t \rangle - \min_{\pi \in \Pi} \sum_{t=|\bar{T}_i|+1}^{|\bar{T}_{i+1}|} \langle \pi, l_t \rangle \leq \alpha^i_{|T_i|} \quad (6)$$

From Inequalities (5) and (6) we have:

$$\sum_{t=|\bar{T}_i|+1}^{|\bar{T}_{i+1}|} \langle \pi_t, l_t \rangle - \min_{\pi \in \Pi} \sum_{t=|\bar{T}_i|+1}^{|\bar{T}_{i+1}|} \langle \pi, l_t \rangle \leq \sqrt{\frac{|T_i|}{2} \log(i)} + \alpha^i_{|T_i|} \quad (7)$$

Sum up the inequality (7) for $i = 1, \ldots k$ we have:

$$\sum_{i=1}^{k} \left( \sqrt{\frac{|T_i|}{2} \log(i)} + \alpha^i_{|T_i|} \right) \geq \sum_{t=1}^{T} \langle \boldsymbol{\pi}_t, \boldsymbol{l}_t \rangle - \sum_{i=1}^{k} \min_{\boldsymbol{\pi} \in \Pi} \sum_{t=|\bar{T}_i|+1}^{|\bar{T}_{i+1}|} \langle \boldsymbol{\pi}, \boldsymbol{l}_t \rangle$$

$$\geq \sum_{t=1}^{T} \langle \boldsymbol{\pi}_t, \boldsymbol{l}_t \rangle - \min_{\boldsymbol{\pi} \in \Pi} \sum_{i=1}^{k} \sum_{t=|\bar{T}_i|+1}^{|\bar{T}_{i+1}|} \langle \boldsymbol{\pi}, \boldsymbol{l}_t \rangle = \sum_{t=1}^{T} \langle \boldsymbol{\pi}_t, \boldsymbol{l}_t \rangle - \min_{\boldsymbol{\pi} \in \Pi} \sum_{i=1}^{T} \langle \boldsymbol{\pi}, \boldsymbol{l}_t \rangle \quad (8a)$$

$$\implies \sqrt{\frac{Tk \log(k)}{2}} + \sum_{i=1}^{k} \alpha^i_{|T_i|} \geq \sum_{t=1}^{T} \langle \boldsymbol{\pi}_t, \boldsymbol{l}_t \rangle - \min_{\boldsymbol{\pi} \in \Pi} \sum_{i=1}^{T} \langle \boldsymbol{\pi}, \boldsymbol{l}_t \rangle. \quad (8b)$$

Inequality (8a) is due to $\sum \min \leq \min \sum$. Inequality (8b) comes from Cauchy-Schwarz inequality and Stirling' approximation. Thus, we have the derived regret bound. $\square$

## B  Online Double Oracle

### B.1  Proof of Theorem 5

**Theorem.** *Suppose both players apply OSO. Let $k_1$, $k_2$ denote the size of effective strategy set for each player. Then, the average strategies of both players converge to the NE with the rate:*

$$\epsilon_T = \sqrt{\frac{k_1 \log(k_1)}{2T}} + \sqrt{\frac{k_2 \log(k_2)}{2T}}.$$

*In situation where both players follow OSO with Less-Frequent Best Response in Equation (6) and $\alpha^i_{t-|\bar{T}_i|} = \sqrt{t - |\bar{T}_i|}$, the convergence rate to NE will be*

$$\epsilon_T = \sqrt{\frac{k_1 \log(k_1)}{2T}} + \sqrt{\frac{k_2 \log(k_2)}{2T}} + \frac{\sqrt{k_1} + \sqrt{k_2}}{\sqrt{T}}.$$

*Proof.* Using the regret bound of OSO algorithm in Theorem 4 we have:

$$\sum_{t=1}^{T} \boldsymbol{\pi}_t^\top \boldsymbol{A} \boldsymbol{c}_t - \min_{\boldsymbol{\pi} \in \Pi} \sum_{t=1}^{T} \boldsymbol{\pi}^\top \boldsymbol{A} \boldsymbol{c}_t \leq \sqrt{\frac{Tk_1 \log(k_1)}{2}},$$

$$\max_{\boldsymbol{c} \in C} \sum_{t=1}^{T} \boldsymbol{\pi}_t^\top \boldsymbol{A} \boldsymbol{c} - \sum_{t=1}^{T} \boldsymbol{\pi}_t^\top \boldsymbol{A} \boldsymbol{c}_t \leq \sqrt{\frac{Tk_2 \log(k_2)}{2}}.$$

From the above inequalities we can derive that

$$\bar{\boldsymbol{\pi}}^\top \boldsymbol{A} \bar{\boldsymbol{c}} \geq \min_{\boldsymbol{\pi} \in \Pi} \boldsymbol{\pi}^\top \boldsymbol{A} \bar{\boldsymbol{c}} \geq \frac{1}{T} \sum_{t=1}^{T} \boldsymbol{\pi}_t^\top \boldsymbol{A} \boldsymbol{c}_t - \sqrt{\frac{k_1 \log(k_1)}{2T}}$$

$$\geq \max_{\boldsymbol{c} \in C} \sum_{t=1}^{T} \bar{\boldsymbol{\pi}}^\top \boldsymbol{A} \boldsymbol{c} - \sqrt{\frac{k_2 \log(k_2)}{2T}} - \sqrt{\frac{k_1 \log(k_1)}{2T}}.$$

Similarly, we have

$$\bar{\boldsymbol{\pi}}^\top \boldsymbol{A} \bar{\boldsymbol{c}} \leq \max_{\boldsymbol{c} \in C} \bar{\boldsymbol{\pi}}^\top \boldsymbol{A} \boldsymbol{c} \leq \frac{1}{T} \sum_{t=1}^{T} \boldsymbol{\pi}_t^\top \boldsymbol{A} \boldsymbol{c}_t + \sqrt{\frac{k_2 \log(k_2)}{2T}}$$

$$\leq \min_{\boldsymbol{\pi} \in \Pi} \boldsymbol{\pi}^\top \boldsymbol{A} \bar{\boldsymbol{c}} + \sqrt{\frac{k_1 \log(k_1)}{2T}} + \sqrt{\frac{k_2 \log(k_2)}{2T}}.$$

Thus, with $\epsilon_T = \sqrt{\frac{k_1 \log(k_1)}{2T}} + \sqrt{\frac{k_2 \log(k_2)}{2T}}$ we have

$$\max_{\boldsymbol{c} \in C} \sum_{t=1}^{T} \bar{\boldsymbol{\pi}}^\top \boldsymbol{A} \boldsymbol{c} - \epsilon_T \leq \bar{\boldsymbol{\pi}}^\top \boldsymbol{A} \bar{\boldsymbol{c}} \leq \min_{\boldsymbol{\pi} \in \Pi} \boldsymbol{\pi}^\top \boldsymbol{A} \bar{\boldsymbol{c}} + \epsilon_T.$$

By definition, $(\bar{\boldsymbol{\pi}}, \bar{\boldsymbol{c}})$ is $\epsilon_T$-Nash equilibrium.

In situation where both players follow OSO with Less-Frequent Best Response, following Theorem in A.3 we have:

$$\sum_{t=1}^{T} \boldsymbol{\pi}_t^\top \boldsymbol{A} \boldsymbol{c}_t - \min_{\boldsymbol{\pi} \in \Pi} \sum_{t=1}^{T} \boldsymbol{\pi}^\top \boldsymbol{A} \boldsymbol{c}_t \leq \sqrt{\frac{T k_1 \log(k_1)}{2}} + \sum_{i=1}^{k_1} \alpha_{|T_i|}^i \leq \sqrt{\frac{T k_1 \log(k_1)}{2}} + \sqrt{k_1}\sqrt{T}$$

$$\max_{\boldsymbol{c} \in C} \sum_{t=1}^{T} \boldsymbol{\pi}_t^\top \boldsymbol{A} \boldsymbol{c} - \sum_{t=1}^{T} \boldsymbol{\pi}_t^\top \boldsymbol{A} \boldsymbol{c}_t \leq \sqrt{\frac{T k_2 \log(k_2)}{2}} + \sum_{i=1}^{k_2} \alpha_{|T_i|}^i \leq \sqrt{\frac{T k_2 \log(k_2)}{2}} + \sqrt{k_2}\sqrt{T}.$$

Thus, using the same above arguments in the case of OSO, with $\epsilon_T = \sqrt{\frac{k_1 \log(k_1)}{2T}} + \sqrt{\frac{k_2 \log(k_2)}{2T}} + \frac{\sqrt{k_1} + \sqrt{k_2}}{\sqrt{T}}$, we have:

$$\max_{\boldsymbol{c} \in C} \sum_{t=1}^{T} \bar{\boldsymbol{\pi}}^\top \boldsymbol{A} \boldsymbol{c} - \epsilon_T \leq \bar{\boldsymbol{\pi}}^\top \boldsymbol{A} \bar{\boldsymbol{c}} \leq \min_{\boldsymbol{\pi} \in \Pi} \boldsymbol{\pi}^\top \boldsymbol{A} \bar{\boldsymbol{c}} + \epsilon_T.$$

The proof is complete. $\square$

### B.2 Proof of Lemma 2

**Lemma.** *DO will converge to $\epsilon$-NE if players can only access to an $\epsilon$-best response in each round.*

*Proof.* We first prove in the case of single oracle algorithm. The double oracle proof will be similar. Since the number of strategies is finite, by the same argument in the the case of exact best response, the process will converge. Suppose that at time step $t$, the process stops. Since we use $\epsilon$-best response, we have the following relationship:

$$\boldsymbol{\pi}_t^\top \boldsymbol{A}_t \boldsymbol{l}_t - \min_{\boldsymbol{\pi} \in \Pi} \boldsymbol{\pi}^\top \boldsymbol{A} \boldsymbol{l}_t \leq \epsilon$$

If we set the weight of pure strategies does not appear in $\boldsymbol{\pi}_t$ to be zero to make a $\boldsymbol{\pi}_t'$, then it is obvious that

$$\boldsymbol{\pi}_t^\top \boldsymbol{A}_t = \boldsymbol{\pi}_t'^\top \boldsymbol{A}$$

Thus, we have the following relationship:

$$\boldsymbol{\pi}_t'^\top \boldsymbol{A} \boldsymbol{l}_t - \min_{\boldsymbol{\pi} \in \Pi} \boldsymbol{\pi}^\top \boldsymbol{A} \boldsymbol{l}_t \leq \epsilon \quad (9)$$

Further, since $\boldsymbol{l}_t$ is Nash equilibrium of $A_t$, we also have

$$\max_{\boldsymbol{l} \in \Delta_l} \boldsymbol{\pi}_t'^\top \boldsymbol{A} \boldsymbol{l} - \boldsymbol{\pi}_t'^\top \boldsymbol{A} \boldsymbol{l}_t = \max_{\boldsymbol{l} \in \Delta_l} \boldsymbol{\pi}_t^\top \boldsymbol{A}_t \boldsymbol{l} - \boldsymbol{\pi}_t^\top \boldsymbol{A}_t \boldsymbol{l}_t = 0 \quad (10)$$

From inequalities (9) and (10), by definition we conclude that $(\boldsymbol{\pi}_t', \boldsymbol{l}_t)$ is $\epsilon$-NE of the game $\boldsymbol{A}$. $\square$

### B.3 Proof of Theorem 6

**Theorem.** *Suppose OSO agent can only access the $\epsilon$-best response in each iteration when following Algorithm 2, if the adversary follows a no-regret algorithm, then the average strategy of the agent will converge to an $\epsilon$-NE. Furthermore, the algorithm is $\epsilon$-regret:*

$$\lim_{T \to \infty} \frac{R_T}{T} \leq \epsilon; \quad R_T = \max_{\boldsymbol{\pi} \in \Delta_\Pi} \sum_{t=1}^{T} \left( \boldsymbol{\pi}_t^\top \boldsymbol{A} \boldsymbol{c}_t - \boldsymbol{\pi}^\top \boldsymbol{A} \boldsymbol{c}_t \right).$$

*Proof.* Suppose that the player uses the Multiplicative Weights Update in Algorithm 2 with $\epsilon$-best response. Denote $|T_1|, |T_2|, \ldots, |T_k|$ be the lengths of time windows during each of which the subset of strategies the no-regret algorithm considers does not change. Furthermore,

$$\sum_{i=1}^{k} |T_k| = T.$$

In a time window $T_i$, the regret with respect to the best fixed strategy in the effective strategy set is:

$$\sum_{t=|\bar{T}_i|+1}^{|\bar{T}_{i+1}|} \langle \boldsymbol{\pi}_t, \boldsymbol{l}_t \rangle - \min_{\boldsymbol{\pi} \in \Pi_{|\bar{T}_i|+1}} \sum_{t=|\bar{T}_i|+1}^{|\bar{T}_{(i+1)}|} \langle \boldsymbol{\pi}, \boldsymbol{l}_t \rangle \leq \sqrt{\frac{|T_i|}{2} \log(i)}, \tag{11}$$

where $|\bar{T}_i| = \sum_{j=1}^{i-1} |T_j|$. Since in the time window $T_i$, the $\epsilon$-best response strategy stays in $\Pi_{\bar{T}_i+1}$ and therefore we have:

$$\min_{\boldsymbol{\pi} \in \Pi_{|\bar{T}_i|+1}} \sum_{t=|\bar{T}_i|+1}^{|\bar{T}_{(i+1)}|} \langle \boldsymbol{\pi}, \boldsymbol{l}_t \rangle - \min_{\boldsymbol{\pi} \in \Pi} \sum_{t=|\bar{T}_i|+1}^{|\bar{T}_{(i+1)}|} \langle \boldsymbol{\pi}, \boldsymbol{l}_t \rangle \leq \epsilon |T_i|$$

Then, from the inequality (11) we have:

$$\sum_{t=|\bar{T}_i|+1}^{|\bar{T}_{(i+1)}|} \langle \boldsymbol{\pi}_t, \boldsymbol{l}_t \rangle - \min_{\boldsymbol{\pi} \in \Pi} \sum_{t=|\bar{T}_i|+1}^{|\bar{T}_{(i+1)}|} \langle \boldsymbol{\pi}, \boldsymbol{l}_t \rangle \leq \sqrt{\frac{|T_i|}{2} \log(i)} + \epsilon |T_i|, \tag{12}$$

Sum up the inequality (12) for $i = 1, \ldots k$ we have:

$$\sum_{t=1}^{T} \langle \boldsymbol{\pi}_t, \boldsymbol{l}_t \rangle - \sum_{i=1}^{k} \min_{\boldsymbol{\pi} \in \Pi} \sum_{t=|\bar{T}_i|+1}^{|\bar{T}_{(i+1)}|} \langle \boldsymbol{\pi}, \boldsymbol{l}_t \rangle \leq \sum_{i=1}^{k} \sqrt{\frac{|T_i|}{2} \log(i)} + \epsilon |T_i|,$$

$$\Rightarrow \sum_{t=1}^{T} \langle \boldsymbol{\pi}_t, \boldsymbol{l}_t \rangle - \min_{\boldsymbol{\pi} \in \Pi} \sum_{i=1}^{k} \sum_{t=|\bar{T}_i|+1}^{|\bar{T}_{(i+1)}|} \langle \boldsymbol{\pi}, \boldsymbol{l}_t \rangle \leq \epsilon T + \sum_{i=1}^{k} \sqrt{\frac{|T_i|}{2} \log(i)} \tag{13a}$$

$$\Longrightarrow \sum_{t=1}^{T} \langle \boldsymbol{\pi}_t, \boldsymbol{l}_t \rangle - \min_{\boldsymbol{\pi} \in \Pi} \sum_{t=1}^{T} \langle \boldsymbol{\pi}, \boldsymbol{l}_t \rangle \leq \epsilon T + \sum_{i=1}^{k} \sqrt{\frac{|T_i|}{2} \log(i)}$$

$$\Longrightarrow \sum_{t=1}^{T} \langle \boldsymbol{\pi}_t, \boldsymbol{l}_t \rangle - \min_{\boldsymbol{\pi} \in \Pi} \sum_{t=1}^{T} \langle \boldsymbol{\pi}, \boldsymbol{l}_t \rangle \leq \epsilon T + \sqrt{\frac{T}{2}} \sqrt{k \log(k)}. \tag{13b}$$

Inequality (13a) is due to $\sum \min \leq \min \sum$. Inequality (13b) comes from Cauchy-Schwarz inequality and Stirling' approximation. Using inequality (13b), we have:

$$\min_{\boldsymbol{\pi} \in \Pi} \langle \boldsymbol{\pi}, \bar{l} \rangle \geq \frac{1}{T} \sum_{t=1}^{T} \langle \boldsymbol{\pi}_t, \boldsymbol{l}_t \rangle - \sqrt{\frac{k \log(k)}{2T}} - \epsilon. \tag{14}$$

That is, the OSO algorithm is $\epsilon$-regret in the case of $\epsilon$-best response.

Since the adversary follows a no-regret algorithm, we have:

$$\max_{l \in \Delta_L} \sum_{t=1}^{T} \langle \boldsymbol{\pi}_t, l \rangle - \sum_{t=1}^{T} \langle \boldsymbol{\pi}_t, \boldsymbol{l}_t \rangle \leq \sqrt{\frac{T}{2}} \sqrt{\log(L)}$$

$$\Longrightarrow \max_{l \in \Delta_L} \langle \bar{\boldsymbol{\pi}}, l \rangle \leq \frac{1}{T} \sum_{t=1}^{T} \langle \boldsymbol{\pi}_t, \boldsymbol{l}_t \rangle + \sqrt{\frac{\log(L)}{2T}} \tag{15}$$

Using the inequalities in (14) and (15) we have:

$$\langle \bar{\boldsymbol{\pi}}, \bar{l} \rangle \geq \min_{\boldsymbol{\pi} \in \Pi} \langle \boldsymbol{\pi}, \bar{l} \rangle \geq \frac{1}{T} \sum_{t=1}^{T} \langle \boldsymbol{\pi}_t, \boldsymbol{l}_t \rangle - \sqrt{\frac{k \log(k)}{2T}} - \epsilon$$

$$\geq \max_{l \in \Delta_L} \langle \bar{\boldsymbol{\pi}}, l \rangle - \sqrt{\frac{\log(L)}{2T}} - \sqrt{\frac{k \log(k)}{2T}} - \epsilon$$

Similarly, we also have:

$$\langle \bar{\boldsymbol{\pi}}, \bar{l} \rangle \leq \max_{l \in \Delta_L} \langle \bar{\boldsymbol{\pi}}, l \rangle \leq \frac{1}{T} \sum_{t=1}^{T} \langle \boldsymbol{\pi}_t, \boldsymbol{l}_t \rangle + \sqrt{\frac{\log(L)}{2T}}$$

$$\leq \min_{\boldsymbol{\pi} \in \Pi} \langle \boldsymbol{\pi}, \bar{l} \rangle + \epsilon + \sqrt{\frac{k \log(k)}{2T}} + \sqrt{\frac{\log(L)}{2T}}$$

17

Take the limit $T \to \infty$, we then have:

$$\max_{l \in \Delta_L} \langle \bar{\pi}, l \rangle - \epsilon \leq \langle \bar{\pi}, \bar{l} \rangle \leq min_{\pi \in \Pi} \langle \pi, \bar{l} \rangle + \epsilon$$

Thus $(\bar{\pi}, \bar{l})$ is the $\epsilon$-Nash equilibrium of the game. □

### B.4 Prod Algorithm with Online Single Oracle

Denote the OSO Algorithm 2 by $A$ and the existing algorithm with regret bound $R_B(T)$ by B. Then we have the following Prod Algorithm [3]:

---

**Algorithm 2:** OSO-Prod algorithm

**Input:** learning rate $\eta \in (0, 1/2]$, initial weights $w_{1,B}, w_{1,A}$, num. of rounds $T$;
**for** $t = 1, 2, \ldots, T$ **do**
  Let $s_t = \frac{w_{t,B}}{w_{t,B} + w_{1,A}}$
  Observe $a_t$ and $b_t$ and calculate $x_t = s_t b_t + (1 - s_t) a_t$
  Observe $l_t$ and suffer loss $l_t(x_t)$
  Feed $l_t$ to $A$ and $B$
  Compute $\delta_t = l_t(a_t) - l_t(b_t)$ and set $w_{t+1,B} = w_{t,B}(1 + \eta \delta_t)$
**end for**

---

**Theorem** (Theorem 7). *Following the OSO-Prod algorithm with the learning rate $\eta = \frac{1}{2}\sqrt{(\log(T)/T)}$ and $w_{1,A} = 1 - w_{1,B} = 1 - \eta$ guarantees the regret bound $R_T$(OSO-Prod) of*

$$\min(R_T(B) + 2\sqrt{T \log(T)}, \frac{T\sqrt{k \log(k)}}{2} + 2\log(2)).$$

The proof of Theorem 7 comes directly from Corollary 1 in [3]. We provide the Corollary here for the completeness of the paper:

**Corollary 1** (Corollary 1 in [3]). *Let $C \geq 1$ be an upper bound on the total benchmark loss $L_T(A)$. Then setting $\eta = \frac{1}{2}\sqrt{\frac{\log(C)}{C}}$ and $w_{1,A} = 1 - w_{1,B} = 1 - \eta$ simultaneously guarantees*

$$R_T((A, B) - Prod) \leq R_T(B) + 2\sqrt{C \log(C)}$$

*and*

$$R_T((A, B) - Prod) \leq R_T(A) + 2\log(2).$$

Table 2: Size of the Nash Support of Games

| Game | Total Strategies | Size of Nash support |
|:---:|:---:|:---:|
| 3-Move Parity Game 2 | 160 | 1 |
| 5,4-Blotto | 56 | 6 |
| AlphaStar | 888 | 3 |
| Connect Four | 1470 | 23 |
| Disc Game | 1000 | 27 |
| Elo game + noise=0.1 | 1000 | 6 |
| Elo game | 1000 | 1 |
| Go (boardsize=3,komi=6.5) | 1933 | 13 |
| Misere (game=tic tac toe) | 926 | 1 |
| Normal Bernoulli game | 1000 | 5 |
| Quoridor (boardsize=3) | 1404 | 1 |
| Random game of skill | 1000 | 5 |
| Tic Tac Toe | 880 | 1 |
| Transitive game | 1000 | 1 |
| Triangular game | 1000 | 1 |

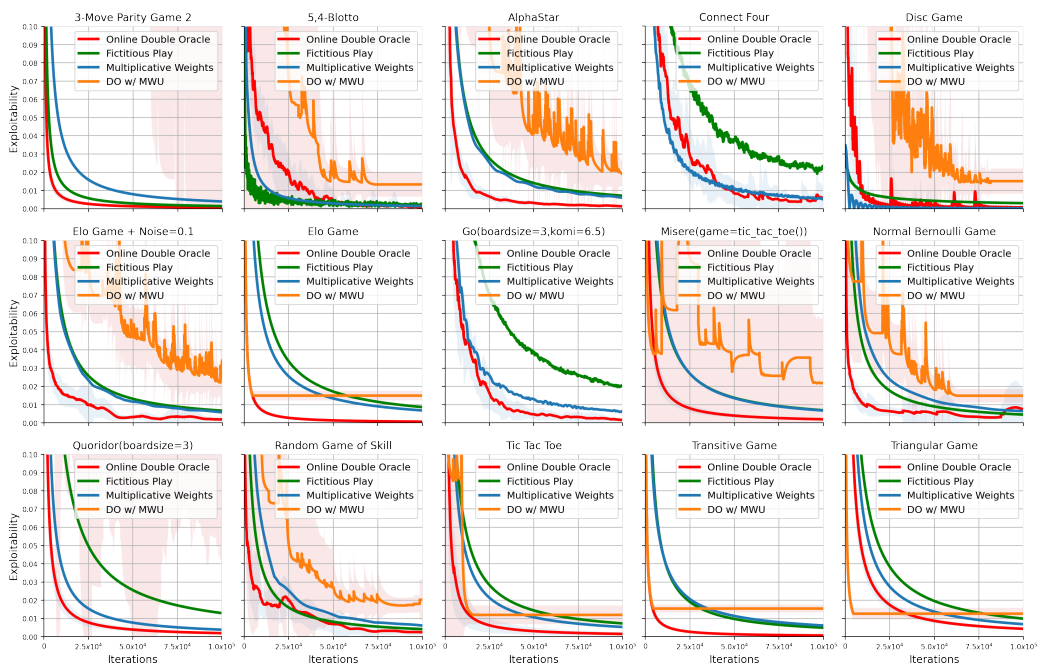

Figure 1: Performance comparisons under self-plays

## C  Additional Experimental Results

We provide further experiments to demonstrate the performance of the Online Single Oracle and Online Double Oracle algorithms. **All the codes for the experiments are attached as supplementary.**

**Performance on Real-World Zero-Sum Games.** We report the full results on the 15 popular real-world games on exploitability and average payoff in Figure 1 and Figure 2, respectively.

19

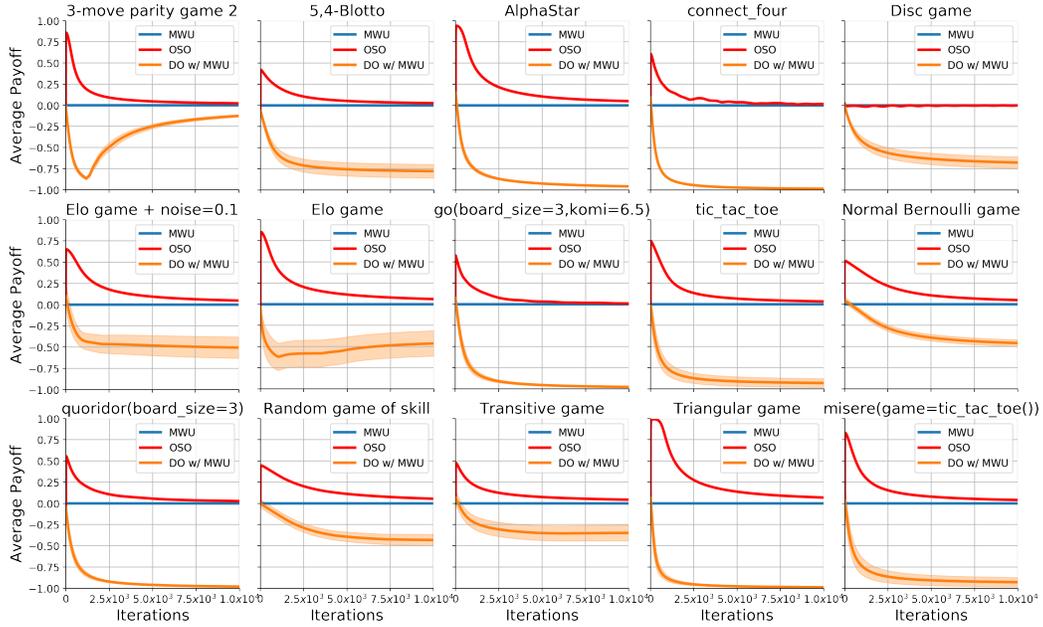

Figure 2: Payoffs of playing against an MWU adversary

**Small Support Size of NEs in Real-World Zero-Sum Games.** Along with the theoretical guarantee in Lemma 1, we demonstrate that Assumption 1 holds true for many real-world games. To do that, we report the total strategies (i.e., number of pure strategy for each player) and the size of Nash support for 15 popular real-world games [2] in Table 2. As we can see in Table 2, in all of the game we tested, the support size of NEs is a fraction of the game size, reassuring the accuracy of Assumption 1. Notably, in the game with pure NEs (i.e., size of Nash support equals to 1), ODO outperforms other baselines by a large margin, both in exploitability and average payoff.



\end{document}